\documentclass{article}

\usepackage[numbers]{natbib}
\usepackage{amsmath}
\usepackage{amssymb}
\usepackage{mathtools}
\usepackage{amsthm}
\usepackage{arxiv}

\usepackage{algorithm}
\usepackage{algpseudocode}

\usepackage{float}
\theoremstyle{plain}
\newtheorem{theorem}{Theorem}[section]

\newtheorem{lemma}[theorem]{Lemma}

\theoremstyle{definition}
\newtheorem{definition}[theorem]{Definition}

\theoremstyle{remark}
\newtheorem{remark}[theorem]{Remark}

\newcommand{\Reeb}{{\mathrm R}}
\newcommand{\Clus}{\text{Clus}}
\newcommand{\Xset}{{\mathcal X}}
\newcommand{\R}{\mathbb{R}}
\newcommand{\Xs}{\mathbb{X}}
\newcommand{\Ys}{\mathbb{Y}}
\newcommand{\interv}{I}
\newcommand{\Metrsp}{X}
\newcommand{\loss}{\ell}
\newcommand{\PD}{PD}
\newcommand{\MapComp}{\text{MapComp}}
\newcommand{\MapPers}{\text{MapPers}}
\newcommand{\VertexFilt}{\text{VertexFilt}}
\newcommand{\SubFilt}{\text{SubFilt}}
\newcommand{\Persistence}{\text{Persistence}}
\newcommand{\nbMC}{M}

\newcommand{\muXs}{\mu}
\newcommand{\Prob}{\mathbb P}
\newcommand{\ParSpace}{\R^s}

\usepackage[utf8]{inputenc} 
\usepackage[T1]{fontenc}   
\usepackage{hyperref}     
\usepackage{url}            
\usepackage{booktabs}       
\usepackage{amsfonts}       
\usepackage{nicefrac}      
\usepackage{microtype}      
\usepackage{graphicx}
\usepackage{natbib}
\usepackage{doi}

\title{Differentiable Mapper for Topological Optimization of Data Representation}

\author{ Ziyad Oulhaj\\
    Nantes Universit\'e, Centrale Nantes \\
    Laboratoire de Math\'ematiques Jean Leray, CNRS UMR 6629\\
	Nantes, France\\
	\texttt{ziyad.oulhaj@ec-nantes.fr} \\
	\And
	Mathieu Carrière \\
	DataShape\\
	Centre Inria d'Université Côte d'Azur\\
	Sophia Antipolis, France \\
	\texttt{mathieu.carriere@inria.fr} \\
    \And
	Bertrand Michel \\
    Nantes Universit\'e, Centrale Nantes \\
Laboratoire de Math\'ematiques Jean Leray, CNRS UMR 6629\\
	Nantes, France\\
	\texttt{bertrand.michel@ec-nantes.fr} \\
}

\date{}

\renewcommand{\shorttitle}{Mapper Optimization}

\hypersetup{
pdftitle={Differentiable Mapper for Topological Optimization of Data Representation},
pdfsubject={Machine Learning, TDA},
pdfauthor={Ziyad Oulhaj, Mathieu Carrière, Bertrand Michel},
pdfkeywords={ Mapper Graph, Data Visualization, Topological Data Analysis, Persistent Homology, Optimization},
}

\begin{document}
\maketitle

\begin{abstract}
Unsupervised data representation and visualization using tools from topology is an active and growing field 
%at the intersection 
of Topological Data Analysis (TDA) 
%, non-convex optimization 
and data science.
Its most prominent line of work is based on the so-called {\em Mapper graph},
which is a combinatorial graph 
%built from the data set, and 
whose topological structures (connected components, branches, loops) are in correspondence with those of the data itself.
While highly generic and applicable, its use has been hampered so far by the manual tuning of its many
parameters---among these, a crucial
%, which is often required in order to get relevant outputs.
one is the so-called {\em filter}: it is a continuous function whose variations on the data set are the main
ingredient for both building the Mapper representation and assessing the presence and sizes of its topological structures.
However, while a few parameter tuning methods have already been investigated for the other Mapper parameters (i.e., resolution, gain, clustering), 
there is currently no method for tuning the filter itself. In this work, we build on a recently proposed 
optimization framework incorporating topology to provide the first filter optimization scheme for Mapper graphs.
In order to achieve this, we propose a relaxed and more general version of the Mapper graph, whose convergence
properties are investigated. Finally, we demonstrate the usefulness of our approach by optimizing Mapper 
graph representations on several datasets, and showcasing the superiority of the optimized representation over arbitrary ones.
\end{abstract}

% keywords can be removed

\section{Introduction}

%Mapper is an efficient algorithm in the field of TDA, application + math developments.
%Its main weakness is its lack of robustness : many parameters too choose; cf papers.

\paragraph*{Mapper graphs and TDA.} The Mapper graph introduced in \cite{singh2007topological} is an essential tool of Topological Data Analysis (TDA), and has been used many times for visualization purposes
on different types of data, including, but not limited to, single-cell sequencing~\cite{wang2018topological,zechel2014topographical}, neural network architectures~\cite{9605336,joseph2021topological}, or 3D meshes~\cite{wang2020exploration,rosen2018inferring}.
Moreover, its ability to precisely encode (within the graph) the presence and sizes of geometric and topological structures in the data in a mathematically founded way
(through the use of algebraic topology) has also proved beneficial for highlighting subpopulations of interest, which are usually detected as peculiar
topological structures of significant sizes, and identifying the key features that best explain such subpopulations against the rest of the Mapper graph.
This general pipeline has become a key component in, e.g., biological inference in single-cell data sets, where differentiating stem cells can usually
be recovered from branching patterns in the corresponding Mapper graphs~\cite{Rizvi2017}.

\paragraph*{Parameter selection.} However, it has quickly become clear that the Mapper graph is quite sensitive to its parameters, in the
sense that the structure of the graph can vary a lot under (even small) changes of its parameters. As such, several pipelines based on Mapper
graphs actually involve brute force optimization: they first compute a grid of Mapper graphs corresponding to many different sets of parameters,
and then they pick the best one, either by manual inspection or with arbitrary criteria---leading to prohibitive running times.
In order to deal with this issue, several methods have been proposed in the literature for either assessing the statistical robustness of a given Mapper graph with respect to the distribution of the studied dataset~\cite{belchi2020numerical, Brown2021}, or for tuning the Mapper parameters automatically~\cite{carriere2018statistical}. 
Unfortunately, most tuning methods
involve simple heuristics that only work for some, but not all Mapper parameters; in particular, the so-called {\em filter parameter} has never been treated, to the best of our knowledge. This is mostly because it is a general continuous function, and can thus vary in a much wilder parameter space than the other Mapper parameters.

In another line of work, ensemble methods have recently been proposed to combine Mapper graphs over multiple parameter sets, rather than trying to find the best one \cite{kang2021ensemble,fitzpatrick2023ensemble},
%This empirically ensemble methods are known
so as to be able to produce outputs that are more robust. However, this imposes to aggregate families of completely different filter functions, with no guarantees on the resulting graph. %will lead  to a relevant final Mapper. 
In this work, we follow a different approach, and rather attempt at identifying an "optimal" filter function by minimizing specific loss functions.

Another approach related to our work is \cite{bui2020f}, where an alternative way of constructing Mapper graphs is proposed using a fuzzy clustering algorithm. %Our framework shares with F-Mapper 
Even though we also adopt a probabilistic approach (that allows, e.g., a point to belong to disconnected intervals in the cover of the filter range), the underlying probabilistic formalism 
that we use is new, while there is none in \cite{bui2020f}. In particular, we introduce stochastic {\it assignment schemes} and we address the parameter selection problem within this framework.

\paragraph*{Contributions.} Our contribution is three-fold:
\begin{itemize}
    \item We introduce {\em Soft Mapper}: a generalization of the combinatorial Mapper graph in the form of a probability distribution on Mapper graphs,
% define a generalized Mapper as a distribution on Mapper (attention,  c'est bien defini ? cf mesurabilité dans papier avec Steve ???), which provides a smoothed  version of the original Mapper we call the Soft Mapper. 
    \item We propose a filter optimization framework adapted to a smooth Soft Mapper distribution with provable convergence guarantees,
%In this framework, we  propose an optimization method to find a filter function
    \item We implement and showcase the efficiency of Mapper filter optimization through Soft Mapper on various data sets, with public, open-source code in \texttt{TensorFlow}. 
\end{itemize}

The following of the article is organized as follows: in Section~\ref{sec:background} we recall the basics on the Mapper algorithm, then in Section~\ref{sec:soft_mapper} we detail the Soft Mapper construction, which is the main  focus of this work. We provide several interesting special cases of Soft Mapper in Section~\ref{sec:examples}, before introducing topological losses that are specific to Mapper graphs in Section~\ref{sec:tpological_loss}. We then present our optimization setting, in which a parameterized family of Mapper filter functions is optimized, in Section~\ref{sec:optim}, and we 
apply it on
%finally show the results of using our method in order to derive an optimal filters for 
3-dimensional shapes and single cell RNA-sequencing data in Section~\ref{sec:experiments}. Finally, we discuss the results of this article and present possible future work directions in Section~\ref{sec:discussion}.

\section{Background on Reeb and Mapper graphs}\label{sec:background}

%\paragraph*{Reeb and Mapper graphs.} 
\paragraph*{Reeb graphs.} Mapper graphs can be understood as numerical approximations of {\em Reeb graphs}, that we now define.
Let $X$ be a topological space and let $f\,:\,X\rightarrow \mathbb{R}$ be a continuous %(sometimes Morse-type) 
function called \emph{filter function}. Let $\sim_{f}$ be the equivalence relation between two elements $x$ and $y$ in $X$ defined by: $x\sim_{f} y$ if and only if $x$ and $y$ are in the same connected component of $f^{-1}(z)$ for some $z$ in $f(X)$. 
%We then define the Reeb graph as:
 The Reeb graph $\Reeb_f(X)$ of  $X$ %computed with a filter function $f$ 
is then simply defined as the quotient space $X/\sim_{f}$.

\paragraph*{Mapper graphs.} The Mapper was introduced in \cite{singh2007topological} as a discrete and computable version of the Reeb graph $\Reeb_f(\Xset)$. 
%in the sense  that it is a discrete and computable approximation of the Reeb graph computed with some filter function. 
%Indeed, it can be seen as a pixelized version of the Reeb graph. 
Assume that we are given a point cloud 
$\Xs_n= \{X_1,\dots,X_n\}\subseteq \Xset$ with known pairwise dissimilarities, as well as  
a filter function $f$ %is chosen and can be 
computed on each point of  $\Xs_n$. The Mapper graph can then be computed with the following
generic version of the Mapper algorithm:
% on $\Xs_n$ computed with the filter function $f$ can be summarized as follows:
\begin{enumerate}
\item Cover the range of values $\Ys_n = f(\Xs_n)$ with a set of consecutive intervals $\interv_1,\dots, \interv_r$  that overlap, i.e., one has $I_i\cap I_{i+1}\neq \varnothing$ for all $1\leq i \leq r-1$.
\item Apply a clustering algorithm to each pre-image $f^{-1}(\interv_j)$, $j\in\{1,...,r\}$. This defines a {\em pullback cover}
$\mathcal{C}=\{\mathcal{C}_{1,1},\dots,\mathcal{C}_{1,k_1},\dots,\mathcal C_{r,1},\dots,\mathcal{C}_{r,k_r}\}$ of %the point cloud 
$\Xs_n$. 
\item The Mapper graph is defined as the {\em nerve} of $\mathcal{C}$. Each node $v_{j,k}$ of the Mapper graph corresponds to an element $\mathcal{C}_{j,k}$ of $\mathcal C$, 
and two nodes $v_{j,k}$ and $v_{j',k'}$ are connected by an edge if and only if  $\mathcal{C}_{j,k} \cap \mathcal{C}_{j',k'} \neq \varnothing$.
\end{enumerate}

%%%%%%%%%%%%%%%%%%%%%%%%%%%%%%%%%%%%%%%%%%%%%%%%%%%%%%%%%%%%%%%%%%%%%%%%%%%%%%%%%%%%%%%%%%%%%%%%%%%%
\section{Soft Mapper construction}\label{sec:soft_mapper}
%%%%%%%%%%%%%%%%%%%%%%%%%%%%%%%%%%%%%%%%%%%%%%%%%%%%%%%%%%%%%%%%%%%%%%%%%%%%%%%%%%%%%%%%%%%%%%%%%%%%

In this section, we introduce our new construction {\em Soft Mapper}, which generalizes Mapper graphs and can be used for non-convex optimization.
In order to do so, we first
%We start by giving a 
provide a
general formalization of the Mapper construction that 
%allows us to avoid  introducing 
does {\em not} require
overlapping intervals and filter functions. 
Then, we use this formalization to 
%construct a generalized Mapper called 
define \emph{Soft Mapper}, which essentially consists in a distribution on regular Mapper graphs. 

\subsection{Mapper graphs built on latent cover assignments}

Let $\Xs_n=\{x_1,...,x_n\}$ be a point cloud lying in a metric space $(\Metrsp,d)$ and let $r \in \mathbb N^\star$. 
%A standard approach is to consider 
For instance, $\Xs_n$  can be obtained from sampling $\Metrsp^n$ according to some distribution $\muXs$.
Then, let $\Clus$ be a clustering algorithm on $(\Metrsp,d)$, that is assumed to be fixed in the following of this work.

\paragraph*{Latent cover assignments.} Any binary matrix $e \in {\{0,1\}}^{n\times r}$ is then called an \emph{r-latent cover assignment} of $\Xs_n$, where $e_{i,j}=1$ must be understood as point $x_i$ belonging to the $j$-th element of a \emph{latent cover} of the data.  For instance, in the standard version of Mapper presented in Section~\ref{sec:background}, the latent cover is obtained from a family of pre-images of intervals that cover the range of the filter function.

The procedure to compute a Mapper graph given an $r$-latent cover assignment 
%$(e_{i,j})_{\substack{1\leq i\leq n \\ 1\leq j\leq r}}$ 
$e \in {\{0,1\}}^{n\times r}$
is straightforward: simply replace $f^{-1}(\interv_j)$ by $\{x_i\,:\,e_{i,j}=1\}$ in the generic Mapper algorithm in Section~\ref{sec:background}, then derive the pullback cover using the clustering algorithm $\Clus$, and finally compute the Mapper graph as the nerve of the pullback cover.

%\begin{enumerate}
%    \item For every $j\in\{1,...,r\}$, define $A_j\subseteq\mathbb{X}_n$ as $\{x_i\,:\,e_{i,j}=1\}$,
%    \item Apply a clustering algorithm to each $A_j$, $j\in\{1,...,r\}$. This produces a cover $\mathcal{C}=\{C_{1,1},...,C_{1,K_1},...,C_{r,1},...,C_{r,K_r}\}$ of $\mathbb{X}_n$ called the \emph{pullback cover}, where $C_{j,k}$ is the $k$-th cluster of $I_j$,
%    \item The Mapper graph is the $1$-skeleton of the nerve complex of $\mathcal{C}$. 
%    It can be seen as a graph with a vertex $v_{j,k}$ for each $C_{j,k}$, and an edge between two vertices $v_{j,k}$ and $v_{j',k'}$ if and only if $C_{j,k}\cap C_{j',k'}\neq \emptyset$. 
%\end{enumerate}

\paragraph*{Mapper function.} Let $\mathbb{K}$ be the set of simplicial complexes of dimension less or equal to $1$ (i.e., graphs) and such that their sets of vertices (i.e., their $0-$skeletons) are subsets of the power set $\mathcal{P}(\Xs_n)$. We define the Mapper complex generating function as:
$$
\MapComp \colon {\{0,1\}}^{n\times r}\longrightarrow \mathbb{K}, 
$$
where $\MapComp$ takes a latent cover assignment as input and creates the corresponding Mapper graph. %associated to $\Xs_n$. 

\subsection{Cover assignment scheme and Soft Mapper}
\label{sec:CovAssign}

Now, we define stochastic schemes for generating latent cover assignments, that we call {\em cover assignment schemes}.

%Hereafter, we consider that that $\Xs$ is a random sample in $\Metrsp^n$, according to a joint distribution $\mathbb{P}_\Xs$. 
%\bert{J'ai retiré le fait que $Xs$ serait random car à aucun moment on n'utilise la loi de  la mention la loi de }
 
\begin{definition}
\label{def:cas}
A \emph{cover assignment scheme} is a double indexed sequence of random variables 
$$A=(A_{i,j})_{\substack{1\leq i\leq n \\ 1\leq j\leq r}}$$ 
%where for each $(i,j)\in\{1,...,n\}\times\{1,...,r\}$, 
such that each
$A_{i,j}$ is a Bernoulli random variable conditionally to $\Xs_n$. Let $p_{i,j}(\Xs_n)$ be the parameter of the Bernoulli distribution of $(A_{i,j}|\Xs_n)$, which is thus a function of the point cloud $\Xs_n$.
%the occurrence of the values $\Xs$ (i.e. having $\{0,1\}$ as an image space).
\end{definition}

%\noindent Note that, in the previous definition, we consider that $\Xs$ is a realization a random variable in $\Metrsp^n$. We will denote the conditional probability given the occurrence $\Xs$ as $\mathbb{P}_\Xs$ in the rest of the paper.

Note that, in Definition~\ref{def:cas}, the Bernoulli variables $A_{i,j}$ are not assumed to be independent nor identically distributed. 
Moreover, $p_{i,j}(\Xs_n)$ can depend only on its associated point $x_i$, or on the whole point cloud $\Xs_n$.

\begin{definition}
Let $A$ be a cover assignment scheme. 
%Following Definition~\ref{def:cas}, 
The \emph{Soft Mapper} of $A$ is defined as the associated 
distribution of Mapper complexes, which corresponds to the push forward measure of the distribution of $A$ by the map $\MapComp$. %Note that this distribution is also conditional to $\Xs$.
\end{definition}

%random variable in $\mathbb{K}$ defined by MapperComplex($A$). 

\section{Examples of cover assignment schemes} \label{sec:examples}

We now give example strategies to define cover assignment schemes, beginning with the one that corresponds to the standard Mapper construction defined in Section~\ref{sec:background}.

\subsection{Standard cover assignment scheme}
\label{sec:stdMapper}

%It is possible to link our construction to the classical Mapper graph construction, i.e. the Mapper graph built to approximate the Reeb graph of a continuous function.\\
%\noindent As before, $\Xs$ is a point cloud of size $n$ in  a general metric space$(\Metrsp,d)$. As explained before, a deterministic clustering algorithm has been fixed for defining the function $\MapComp$. 
Let $f\colon\Xs_n \rightarrow \R$ be a filter function  and let $(\interv_j)_{1\leq j\leq r}$ be a finite cover of the image $f(\Xs_n)$ of $f$.
%Let $r'=\sum_{j=1}^r k_j$, where $k_j$ is the number of clusters detected by the fixed clustering algorithm $\Clus$ on $f^{-1}(I_j)$. 
The standard Mapper graph is then defined as $\MapComp(e^\ast)$, where for every $1\leq i \leq n$ and 
%$1\leq j'=(j,k) \leq r'$ (with $1\leq j\leq r$ and $1\leq k\leq k_j$):
$1\leq j \leq r$:
$$e^\ast_{i,j}=1 \text{ if } f(x_i)\in \interv_j.$$ %\text{ and }x_i\in \mathcal C_{j,k}.$$
The cover assignment scheme $A^\ast$, in this case, is such that every entry $A^\ast_{i,j}$ follows a Dirac distribution on $1$ if $f (x_i)\in \interv_j$,
%and $x_i\in \mathcal C_{j,k}$, 
and a Dirac distribution on  $0$ otherwise. In other words, the parameters of the Bernoulli distributions satisfy $p_{i,j} (\Xs_n) = p_{i,j} (x_i) = 1 $ if  $f (x_i)\in \interv_j$,
%and $x_i\in \mathcal C_{j,k}$, 
and $0$ otherwise, that is
\[
    \mathbb{P}(A^\ast=e | \Xs_n)= 
\begin{cases}
    1& \text{ if } e=e^\ast ,\\
    0              & \text{ otherwise.}
\end{cases}
\]
In this degenerated situation, the random variables $A^\ast_{i,j}$ are all independent conditionally to $\Xs_n$, and $A^\ast_{i,j}$ conditionally to $\Xs_n$ is equal to $A^\ast_{i,j}$ conditionally to $x_i$.

%Note also that the different coordinates of $A^\ast$ conditionally follow Dirac distributions in $\{0,1\}$, but we assume that they can be seen as an extension of the Bernoulli distribution when the success probability is either 1 or 0.\\
% The conditional risk of the cover assignment scheme $A^\ast$ is : 
% $$\mathbb{E}_{\mathbb{X}_n,\theta}(\mathcal{L}(A^\ast,\theta))=\mathcal{L}(e^\ast,\theta).$$

\begin{remark}
    An alternative and relevant approach for the standard Mapper graph is to define the intervals $I_j$ via the quantiles of the distribution of $f(\Xs_n)$.  
In this case, the random variables $A^\ast_{i,j}$ do not only depend on $x_i$, but also on the whole point cloud $\Xs_n$.
\end{remark}

\subsection{Smooth relaxation of the standard cover assignment scheme}
\label{sec:Smooth}

Given some $\delta>0$, we can now define a cover assignment scheme $A_\delta$ that approximates the cover assignment scheme $A^\ast$ arising from the standard Mapper graph,
but that also enjoys useful smoothness properties in the optimization setting that we will consider in the next section. %, when $\delta$ is close to zero.\\
Specifically, using the same notations as before, and denoting each element of the cover with $I_j=[a_j,b_j]$, consider, for each $j\in\{1,...,r\}$, the function $q_j\colon X  \longrightarrow [0,1]$ defined with:
\begin{align*}
&x \mapsto 
\begin{cases}
    1,& \text{if } f(x)\in [a_j,b_j]\\
    \exp(1-1/(1-(\frac{a_i-f(x)}{\delta})^2)),& \text{if } f(x)\in [a_j-\delta,a_j]\\
    \exp(1-1/(1-(\frac{f(x)-b_i}{\delta})^2)),& \text{if } f(x)\in [b_j,b_j+\delta]\\
    0,              & \text{otherwise}
\end{cases}
\end{align*}
\noindent
Now, define $A_\delta=(A_{\delta,i,j})_{\substack{1\leq i\leq n \\ 1\leq j\leq r}}$ to be the random variable in $\{0,1\}^{n\times r}$ such that for every $(i,j)\in\{1,...,n\}\times\{1,...,r\}$:
$$ A_{\delta,i,j}\mid \Xs_n \sim \mathcal{B}(q_j(x_i)),$$
%and we take 
with the $ A_{\delta,i,j}$'s being jointly independent conditionally to $\Xs_n$.
As for the standard cover, the Bernoulli parameter $p_{i,j} = q_j(x_i)$ depends on its associated point $x_i$ and also on the chosen filter $f$. 
%This remark will be important in Section~\ref{sec:optim}.

Moreover, notice that for every $x_i\in\Xs_n$ and $j\in\{1,...,r\}$:

$$
    q_j(x_i)\xrightarrow[\delta \to 0]{}\begin{cases}
    1,& \text{if } f(x_i)\in I_j\\
    0,              & \text{otherwise,}
\end{cases}
$$
and this shows that $A_\delta\xrightarrow[\delta \to 0]{\mathcal{L}}A^\ast.$ Note that even though we can approximate the standard Mapper graph in this way, we do not always want to do so. For example, there could be cases where $\delta$ needs to be large enough so as to account for some uncertainty on the bounds of the cover $(\interv_j)_{1\leq j\leq r}$.

%The reason behind considering such an approximation of $A^\ast$ is because it allows for useful smoothness properties in the optimization setting that we will consider in the following section.  
\begin{remark}
Note that the same relaxed construction can be made for a multi-dimensional Mapper, i.e., for filter functions taking values in $\R^d$ \cite{carriere2022statistical}, by making slight adjustments to the definition of $q_j$ %, for every $j\in\{1,...,r\}$, 
using the Euclidean norm.
\end{remark}

An additional example of a possible cover assignment scheme, which does not imply the existence of a filter function, is given in Appendix \ref{apdx:gaussian}.

%%%%%%%%%%%%%%%%%%%%%%%%%%%%%%%%%%%%%%%%%%%%%%%%%%%%%%%%%%%%%%%%%%%%%%%%%%%%%%%%%%%%%%%%%%%%%%%%%%%%
\section{Topological risk of Soft Mappers}\label{sec:tpological_loss}
%%%%%%%%%%%%%%%%%%%%%%%%%%%%%%%%%%%%%%%%%%%%%%%%%%%%%%%%%%%%%%%%%%%%%%%%%%%%%%%%%%%%%%%%%%%%%%%%%%%%

We now switch to the problem of designing filter functions automatically for Mapper graphs using Soft Mapper.
%go back to the initial problem general setting, where $\Xs$ is a point cloud in a  metric space $\Metrsp$. 
%Our aim is to find a relevant filter for computing Mapper graphs. 
To answer this ill-posed problem, we propose to look for filter functions that are optimal with respect to some topological criteria associated to their (Soft)Mapper graphs.
In particular, we focus on topological losses based on {\em persistent homology}.

%Let $\Theta$ be a non-zero integer and consider a parameterized family of functions $\{f_\theta\colon\mathbb{R}^p\longrightarrow\mathbb{R},\,\theta\in\mathbb{R}^{\Theta}\}$.\\
\subsection{Topological signature for Mapper graphs}
\label{sec:MapPers}

\paragraph*{Persistent homology.}
%\ziy{J'ai l'impression qu'on ne peut pas se passer du fait d'avoir une fonction filtre pour définir une filtration sur les Mappers de taille différente d'une manière similaire. On peut toutefois enlever le fait que ça dépend de paramètres. On peut faire en sorte de ne commencer à parler de paramètres qu'à partir de la section "optim". J'ai noté cette fonction F pour qu'on voit qu'elle peut ne pas avoir de rapport avec f qui est le filtre Mapper.}
Persistent homology is a powerful tool that allows to encode the topological information contained in a nested family of simplicial complexes, also called a \emph{filtered simplicial complex}, see for instance \cite{edelsbrunner2022computational} for a general introduction. It traces the evolution of the homology groups of the nested complexes across different scales, producing topological descriptors that are, in particular, useful in machine learning pipelines \cite{chazal2021introduction}.
In the context of Mapper graphs, a variation of persistent homology called {\em extended} persistent homology has been proved useful, as applying it on Mapper graphs produces descriptors called {\em extended persistence diagrams}. These diagrams only require to define a {\em filtration function} on the graph, and are made of points in the Euclidean plane, each point encoding the presence and size (w.r.t. the filtration function) of a particular topological structure of the Mapper graph (such as a connected component, a branch or a loop).
%Regular persistence can be extended by tracking, not only the evolution of the homology groups for an increasing scale, but also for a decreasing one. This extended persistence is practical because it gives information about the "essential" part of regular persistence, i.e. the homology groups that do not disappear in the filtered simplicial complex. In particular, we use it in our Mapper complex framework to track homology groups in dimension 1 (i.e. loops).
See Section 2 of \cite{carriere2020perslay} for a brief introduction to extended persistence and \cite{cohen2009extending} for a detailed presentation. 

We now define a filtration function on Mapper graphs in order to compute extended persistence diagrams. Let  $\mathcal{F}(\Xs_n,\mathbb{R})$ be the space of real valued functions defined on the point cloud $\Xs_n$. 
For a function $F \in \mathcal{F}(\Xs_n,\R)$, we first associate a filtration $\phi$ to some $K \in {\rm im}(\text{MapComp})$ with: 
$$\forall \sigma\in K\,:\,\phi(\sigma)=\max_{c\in\sigma}\frac{\sum_{x\in c} F(x) }{\text{card}(c)},$$
that is, node filtration values are defined as the average filter values of the data points associated to the node, and edge filtration values are computed as the maximum of their node values.  
Then, we compute the extended persistence diagram (which we consider as a subset of 
%$\mathbb{R}^2\times\mathbb{R}$
$\mathbb{R}^2$
) of the filtered simplicial complex $(K,\phi)$. We denote by $\MapPers$ the function that takes a Mapper graph and a scalar function on $\Xs_n$, and then outputs the persistence diagram:
$$
\MapPers \colon \mathbb{K}\times\mathcal{F}(\Xs_n,\R)\longrightarrow \mathcal{P}(\R^2).
$$

% and denoting  such functions, we want to compute the persistent homology of a produced Mapper complex $K$. To this end, we define the following function :

\paragraph*{Persistence specific loss.}
Now, we introduce a generic notation for a loss function---or, more simply, a statistic---that associates a real value to any extended persistence diagram. Denoting $\PD$ as the set of subsets of $\R^2$ consisting of a finite number of points outside the diagonal $\Delta=\{(x,x):x\in\mathbb{R}\}$, such a function can be written as $\loss \colon \PD \longrightarrow \mathbb{R}.$

\subsection{Statistical risk of the topological signature associated to Soft Mapper}
 
We finally compute the loss 
associated to %of the extended persistence diagram of 
a Mapper graph with the function

\begin{align*}
\mathcal{L} \colon& {\{0,1\}}^{n\times r}\times\mathcal{F}(\Xs_n,\R)  \longrightarrow \R\\
&(e,F) \longmapsto \loss \left(\MapPers \left( \MapComp(e),F \right) \right).
\end{align*}

\noindent Then, we define the risk of a Soft Mapper $\MapComp(A)$ by integrating the loss according to the distribution of the Soft Mapper, or equivalently according to the distribution of the cover assignment scheme: 
$$\mathbb{E} \left( \mathcal{L}(A,F) | \Xs_n\right)=\sum_{e\in {\{0,1\}}^{n\times r}}\mathcal{L}(e,F)\cdot\mathbb{P}(A=e |\Xs_n).$$
\noindent

Here, %$\Xs_n$ is seen as a random sample, and thus 
both the distribution of $A$ and the risk are conditional to $\Xs_n$. 
Note that the risk could also be integrated with respect to the distribution  of $\Xs_n$. However, in this article, we only consider the non-integrated version of the risk. 
%\bert{si on faisait du sous echantillonnage ou du bootstrap ce serait different, mais ici j'ai l'impression que l'on travaille exclusivement conditionnellement à $\Xs$}

 \section{Conditional risk optimization with respect to parameters}
\label{sec:optim}

Now that we have properly defined risks associated to Soft Mapper distributions, we study in this section the convergence properties of 
filter optimization schemes minimizing such risks. 
%We now consider the problem of filter optimization for Soft Mapper.
 
\subsection{Problem setting}

%a given filter function is used both for defining the persistence diagram of a Mapper graph (Section \ref{sec:MapPers}), but it can also be used for defining the cover assignments, as in Section~\ref{sec:examples}. 
Let us introduce a parameterized family of functions $\{ f_\theta :   \Xs_n  \rightarrow \R , \, \theta  \in \R^s\}$. 
In order to simplify notations, we assume in the following of the article that the function $F$ used to compute persistence diagrams and the filter function $f_\theta$ used to design cover assignments are the same, $F=f_\theta$. 
Let $A$ be a cover assignment scheme whose joint distribution $\Prob_\theta$ depends on the filter function $f_\theta$; 
%which we use to define the Soft Mapper. 
%in contrast with Section~\ref{sec:CovAssign}, and as in Section~\ref{sec:Smooth}, here the %we have also introduced filters functions and the 
that is the Bernoulli parameters $p_{i,j}$ may depend on the filter function values and the parameters $\theta$. %them (as for the smooth relaxation in Section~\ref{sec:Smooth}).
%the way the Bernoulli parameters $p_{i,j}$ are defined may depend on the filter function (as for the smooth relaxation in Section~\ref{sec:Smooth}).
Note that this dependency is not only true for marginals of the distribution of the cover assignment scheme,
% may depend on $\theta$, 
but also eventually for its dependency structure.

Our goal is to find the optimal set of parameters $\bar \theta$ that minimizes the topological risk associated to  $\MapComp(A$), when $f_\theta$ is used to define the filtration values on the Mapper graphs. 
In other words, if we denote:
\begin{align}
\text{L} \colon \ParSpace  &\longrightarrow \mathbb{R} \notag\\
\theta &\longmapsto \mathbb{E_\theta}(\mathcal{L}(A,f_\theta) | \Xs_n) \label{eq:L},
\end{align}
our aim is to find a minimizer of $\text{L}$. Note that in the definition of $\text{L}$, the expectation depends on $\theta$ because the distribution of $A$ also depends on it.
%cover assignment scheme does depend on it.

In order to prove guarantees about minimizing $\text{L}$, we follow~\cite{carriere2021optimizing}, which uses the theoretical background introduced in~\cite{davis2020stochastic},
in which the authors prove that stochastic gradient descent algorithms converge under certain conditions. 
To use this framework, it suffices to prove two points (see Corollary 5.9. in \cite{davis2020stochastic} and Appendix \ref{apdx:o-minimal}):
 
\begin{itemize}
    \item L is definable in an o-minimal structure, %for elements of o-minimal geometry),
    \item L is locally Lipschitz.
\end{itemize}

\begin{remark}
When the cover assignment scheme is defined as the standard cover assignment scheme corresponding to the standard Mapper graph (see Section~\ref{sec:stdMapper}), %$f_\theta$ is the Mapper filter function, 
this problem amounts to finding an optimal $f_\theta$ that can be used to compute Mapper graphs. We will see however that convergence of the optimization problem in this case is without guarantees, 
which constitutes the main motivation for defining our smooth relaxation Soft Mapper (see Section~\ref{sec:Smooth}).
\end{remark}

\subsection{Theoretical guarantees on the convergence of a gradient descent scheme}

Under regularity assumptions on the parameterized family of filter functions $\mathcal F=\{f_\theta\colon\Xs_n \longrightarrow\mathbb{R},\,\theta\in\ParSpace\}$, we now show that the risk $\text{L}$ in Equation~\eqref{eq:L} is definable and smooth.

\begin{theorem}
\label{thm1}
    Suppose that there exists an o-minimal structure $\mathcal{S}$ such that: 
    \begin{itemize}
        \item for every $x\in\Xs_n$, the function $\theta\mapsto f_\theta(x)$ is definable in $\mathcal{S}$ and is locally Lipschitz,
        \item for every $m\in\mathbb{N}$, the restriction of $\loss$ to the set of (extended) persistence diagrams of size $m$ 
%(by isomorphy considered to be $\mathbb{R}^m$) 
is definable in $\mathcal{S}$ and is locally Lipschitz,
        \item for every $e\in{\{0,1\}}^{n\times r}$, the function $\theta\mapsto \mathbb{P}_\theta(A =e| \Xs_n)$ is definable in $\mathcal{S}$ and is locally Lipschitz.
    \end{itemize}
    Then $\textup{L}$ is definable in $\mathcal{S}$ and is locally Lipschitz.
\end{theorem}

\begin{remark}
    Our proof of Theorem \ref{thm1} is given in Appendix \ref{apdx:proof} in the case where regular persistent homology is used, but it can be extended in a straightforward way to extended persistence diagrams, as extended persistent homology on a simplicial complex $K$ can be equivalently seen as regular persistent homology on the cone on $K$~(see chapter VII.3 in \cite{edelsbrunner2022computational}). Moreover, defining the filtration on the coned complex also extends naturally by using affine transformations.
\end{remark}

Under the assumptions of Theorem \ref{thm1}, it is then possible to give guarantees on the convergence of a stochastic gradient descent scheme to some critical points of $\text{L}$. This only requires additional, but mild and not very restrictive technical conditions regarding the stochastic gradient descent algorithm itself (see Appendix \ref{apdx:sgd}). 
% Namely, if we write the iterates of the algorithm as: 
% $$x_{k+1}=x_k-\alpha_k(y_k+\xi_k),$$
% where 
% $$y_k\in\text{Conv}\left\{\lim_{z\to x_k}\nabla\text{L}(z)\,:\,\text{L is differentiable at }z\right\},$$
% the conditions are as follows~: 
% \begin{enumerate}
%     \item for any $k$, $\alpha_k\geq 0$, $\sum_{k=1}^\infty \alpha_k=+\infty$ and $\sum_{k=1}^\infty \alpha_k^2<+\infty$;
%     \item $\sup_k\Vert x_k\Vert<+\infty$, almost surely;
%     \item Conditionally on the past, $\xi_k$ must have zero mean and have a second moment that is bounded by a function $p\colon\ParSpace\longrightarrow\mathbb{R}$ which is bounded on bounded sets. 
% \end{enumerate}
% Note that the last condition is satisfied by taking a sequence of independent and centered variables with bounded variance, which are also independent of the past iterates $(x_k)_k$ and  $(y_k)_k$.

% According to \cite{davis2020stochastic}, under these three conditions together with the conditions of Theorem~\ref{thm1}, then $(\text{L}(x_k))_k$ converges almost surely to a critical values and the limit points of $(x_k)_k$ are critical points of $L$. 

\subsection{Discussing the assumptions of Theorem \ref{thm1}}
In this section, we discuss the assumptions of Theorem \ref{thm1}, and provide usual cases in which they are satisfied. %, that are usual in practice.

\paragraph*{Assumption 1.} The first assumption concerns the smoothness of the parameterized family of functions $\{f_\theta\colon\Xs_n \longrightarrow\mathbb{R},\,\theta\in\ParSpace\}$ and its regularity with respect to the set of parameters $\theta$. As mentioned before, following the result of \cite{wilkie1996model}, semi-algebraic functions (for example polynomial, rational, minimum and maximum functions), the exponential function and functions defined as compositions and usual operations between them are all definable in a same o-minimal structure. Furthermore, choosing continuously differentiable functions is sufficient to also have the local Lipschitz property.\\
    As such, the family of linear functions $\{f_\theta\colon x\mapsto\langle x,\theta \rangle,\,\theta\in\mathbb{R}^s\}$ satisfies the assumption, as well as the family of parameterized fully-connected neural networks since they are defined by composition between matrix products (which are polynomial) and activation functions involving exponential, maximum and hyperbolic functions.

\paragraph*{Assumption 2.} The second assumption concerns the persistence-based loss $\loss$, that is used to compute the topological risk. In \cite{carriere2021optimizing}, the authors list a number of possible functions for $\loss$ that satisfy our second assumption. For example, $\loss$ can be the \emph{total persistence}, which quantifies the information given by a persistence diagram, defined as:
    %$$\{(u_i,v_i)\}_{1\leq i\leq n_f}\times\{t_j\}_{1\leq j\leq n_e}\longmapsto \sum_{i=1}^{n_f}\vert u_i-v_i\vert.$$
    $$\{(u_i,v_i)\}_{1\leq i\leq n}\longmapsto \sum_{i=1}^{n}\vert u_i-v_i\vert.$$
    It can also be computed from persistence landscapes~\cite{Bubenik2015} or from the bottleneck distance to a target persistence diagram~\cite{carriere2021optimizing}.

\paragraph*{Assumption 3.} Finally, the third assumption concerns the cover assignment scheme $A$. More specifically, it requires the regularity and smoothness of the success probabilities that give the distribution of $A$.\\ 
    Interestingly, this assumption does {\em not} hold for the standard cover assignment scheme. For example, consider the elementary example where $\Xs_n\subset\R$ and $A$ is the standard cover assignment scheme, which is degenerate at $e_\theta$, and which corresponds to the linear filter function $f_\theta\colon x\mapsto  \langle x,\theta\rangle$ and a cover $(\interv_j)$ of its image. Fix a non-zero positive point $x\in\Xs_n$ (a similar argument can be made if it is negative) and a left hand bound $a_j$ of one of the intervals. Denoting $\theta_0=\frac{a_j}{x}$, we have that $\theta\mapsto \mathbb{P}_\theta(A=e_{\theta_0} | \Xs_n)$ is discontinuous at $\theta_0$, since $\forall \epsilon>0\,:\,\langle x,\theta_0-\epsilon\rangle=x\cdot(\theta_0-\epsilon)<a_j,$ and therefore, $\mathbb{P}_{\theta_0-\epsilon}(A=e_{\theta_0} | \Xs_n)=0.$

    % Consider the case where $A_\theta^\ast$ is the classical cover assignment scheme associated to $f_\theta$, which is a function from the same parameterized family we considered above. $A_\theta^\ast$ is a Dirac distribution in $e_\theta^\ast$ such that it is defined above.\\ The third assumption is not automatically verified even though the family of functions is tame, in the sense that it verifies the first assumption. To see this, it suffices to consider a point $\theta_0$ where there exists a point in the dataset $x\in\Xs$ such that $f_{\theta_0}(x)$ is exactly on one of the bounds of the intervals used to cover $f_{\theta_0}(\Xs)$ (always possible for example with a linear function). \\
    % $\theta\mapsto \mathbb{P}_\Xs(A_\theta=e_{\theta_0}^\ast)$ is not continuous (and therefore not locally Lipschitz) around $\theta_0$ because changing $\theta$ even slightly can eject $x$ from one of the intervals making $\mathbb{P}_\Xs(A_\theta=e_{\theta_0}^\ast)$ go from one to zero abruptly.\\
    % \ziy{À retravailler}\\
    This constitutes the main motivation for introducing our smooth cover assignment scheme because the functions $\theta\mapsto \mathbb{P}_\theta (A=e | \Xs_n)$ are in this case products of the functions $q_j$, which are smooth with respect to the parameters (if  our first assumption holds).

\subsection{Computing the conditional risk in practice}

\begin{algorithm}[H]
\caption{Soft Mapper Optimization Algorithm}\label{alg1}
\begin{algorithmic}
\Require Initial parameter set $\theta_0$, Number of Monte Carlo random samples $\nbMC$, Learning rate sequence $(\alpha_i)_i$, Random noise sequence $(\xi_i)_i$, Number of epochs $N$.
\For{$0\leq i \leq N-1$}
\For{$1\leq m \leq \nbMC$}
    \State $e \gets $ sample from $\mathbb{P}_{\theta_i}$ %the conditional distribution of $A$
    \State $y_{i,m} \gets$ an element of the sub-differential in $\theta_i$ of $\mathcal{L}_e\colon\theta\mapsto \mathcal{L}(e,f_\theta)$
\EndFor
\State $y_i \gets \frac{1}{\nbMC}\sum_{m=1}^{\nbMC}y_{i,m}$
\State $\theta_{i+1}\gets \theta_i -\alpha_i(y_i+\xi_i)$
\EndFor
\State {\bf return} $\theta_N$
\end{algorithmic}
\end{algorithm}

Computing the conditional risk L$(\theta)$, for a fixed $\theta\in\ParSpace$, can be costly in practice since it requires computing the loss $\mathcal{L}(e,f_\theta)$ for every possible cover assignment $e\in\{0,1\}^{n\times r}$. As such, we estimate L$(\theta)$ with Monte Carlo methods. Note that this is possible here because the distribution $\mathbb{P}_\theta$  of the cover assignment scheme is indeed explicitly defined and known at each step of the gradient descent. If $\nbMC$ is a non-zero integer and $(e^{(m)})_{1\leq m\leq \nbMC}$ is a sequence of independent realizations of the cover assignment scheme $A$, then the Monte Carlo approximation of the conditional risk is:
$$\Tilde{\text{L}}(\theta)=\frac{1}{\nbMC}\sum_{m=1}^{\nbMC}\mathcal{L}(e^{(m)},f_\theta).$$
The law of large numbers gives: 
$$\Tilde{\text{L}}(\theta)\xrightarrow[\nbMC \to \infty]{a.s}\text{L}(\theta).$$
\noindent Moreover, the coordinates of $A$ follow a Bernoulli conditional distribution, making repeated random sampling straightforward, at least when the marginal distributions of $\mathbb{P}_\theta$ are assumed to be independent. 

For a fixed point cloud $\Xs_n$, a chosen family of parameterized conditional probabilities $\theta \mapsto \mathbb{P}_\theta(\cdot |\Xs_n ) $ and a family of parameterized filters $\theta \mapsto f_\theta$, our corresponding optimization algorithm is detailed in Algorithm \ref{alg1}.

\section{Numerical Experiments}\label{sec:experiments}

%a essayer si le temps le permet:
%-- pour les formes 3D : 
%   + methodes a noyaux (optimiser K.alpha a la place de X.alpha grace au thm du representant)
%   + traitement purement intrinseque avec noyau gaussien geodesique
%   + un petit reseau de neurones
%-- pour le single cell :
%   + coeff devant la penalite L2 pour plateau
%   + KS test gene differentiels (branche de gauche dans la boucle du mapper)
%   + eventuellement un des noyaux de Bertrand, Franck, Anthony
%-- idee de PCA topo avec plusieurs optims de filtres orthogonaux

In this section, we illustrate the efficiency of optimizing filter functions with Soft Mapper.
In particular, we show that Mapper graphs computed from an optimized filter function (computed with
gradient descent on Soft Mapper) are generally much better structured than Mapper graphs obtained
from arbitrary filters (as is usually done in the Mapper applications). We present applications on 
3D shape data in Section~\ref{subsec:expe-3d} and on single-cell RNA sequencing data in Section~\ref{subsec:single-cell}.
% our goal is to show the practical effectiveness of our differentiable version of the Mapper algorithm.
Our code is available in the following \texttt{Github} repository \cite{github}.

\subsection{Mapper optimization on 3D shapes}
\label{subsec:expe-3d}

A first application where we can use the Soft Mapper optimization setting is finding linear filters in order to skeletonize 3-dimensional shapes with Mapper graphs.
Here, our dataset $\Xs_n$ consists each time of a point cloud embedded in $\R^3$. The different point clouds we study are displayed (as meshes) in Figure \ref{fig:3d_shape}.
The parametric family of functions is linear, i.e., equal to $\{f_\theta\colon x\mapsto\langle x,\theta \rangle,\,\theta\in\mathbb{R}^3\}$, and the cover assignment scheme $A_\delta$ is the smooth relaxation of the standard case, with $\delta=10^{-2}\cdot(\max_{x\in\Xs_n}f_\theta(x)-\min_{x\in\Xs_n}f_\theta(x))$. The cover of the image space is given by $r$ intervals of the same length, such that consecutive intervals have a percentage $g$ of their length in common. The clustering algorithm for the three shapes is KMeans. The values of $r$ (also called resolution), $g$ (also called gain) and the number of clusters in the KMeans algorithm, for each 3-dimensional shape, are summarized in Table \ref{tab:r_g} of Appendix~\ref{apdx:experiments}.

\paragraph*{Objective.} Intuitively, the optimal directions to filter the 3-dimensional shapes (in a topological sense) are:
\begin{itemize}
    \item for the human: the vertical direction,
    \item for the octopus: the parallel direction to its tentacles,
    \item for the table: the perpendicular direction to its upper surface,
\end{itemize}
as these directions induce Mapper graphs with more topological structures.
We will therefore measure the quality of our results by comparing our optimized directions $\Bar{\theta}$ to the ones cited above. To find $\Bar{\theta}$, we use the total (regular) persistence as a persistence specific loss $\loss$ and we run Algorithm \ref{alg1} with $N=200$ and $M=10$, each time taking the diagonal as the initial direction, i.e. $\theta_0=(\frac{1}{\sqrt{3}},\frac{1}{\sqrt{3}},\frac{1}{\sqrt{3}})^T$.
The learning curves for each 3-dimensional shape are displayed in Figure \ref{fig:3d_curves}, and the correlations between the optimized directions and those we identified as intuitively optimal are summarized in Table \ref{tab:3d_correlations}. As one can see from the table, we are able to recover these intuitive directions with gradient descent.

\paragraph*{Qualitative assessment.} One can see, in Figures \ref{fig:3d_initial} and \ref{fig:3d_final}, that the regular Mapper graphs built with the initial and final (optimized) filter functions show clear improvement in the ability of the graphs to act as skeletons of the original point clouds. As such, we see that optimizing the Soft Mapper corresponding to the smooth relaxation of the standard cover assignment scheme succeeds in identifying optimal filter functions. The third shape, representing a table, is particularly interesting. Indeed, the optimal direction that we captured is different from the first and the second principal components computed by PCA, since the principal plane of the point cloud is given by the surface of the table. Hence, there is a contrast between the topological criteria that we use, which is the total persistence, and the maximum variance criteria used by PCA.

\begin{figure}[H]

\centering
\includegraphics[width=.3\textwidth]{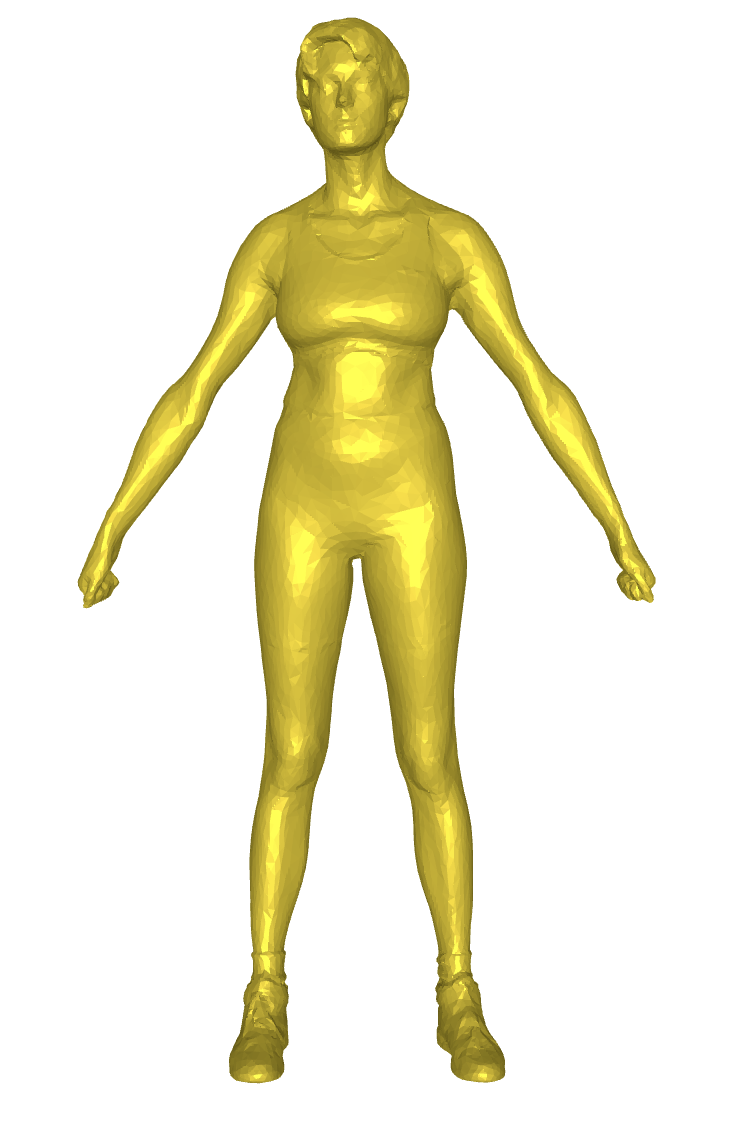}\hfill
\includegraphics[width=.3\textwidth]{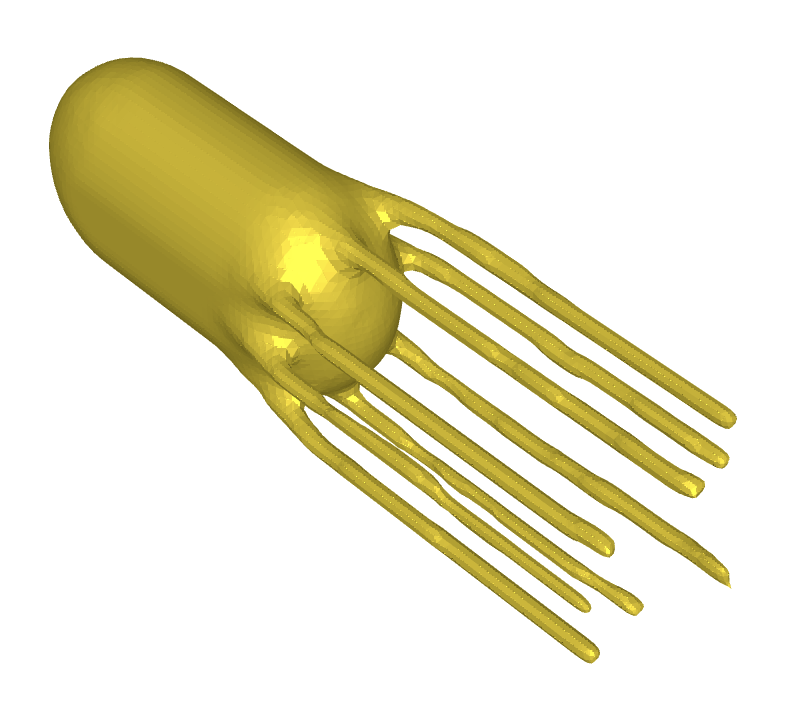}\hfill
\includegraphics[width=.3\textwidth]{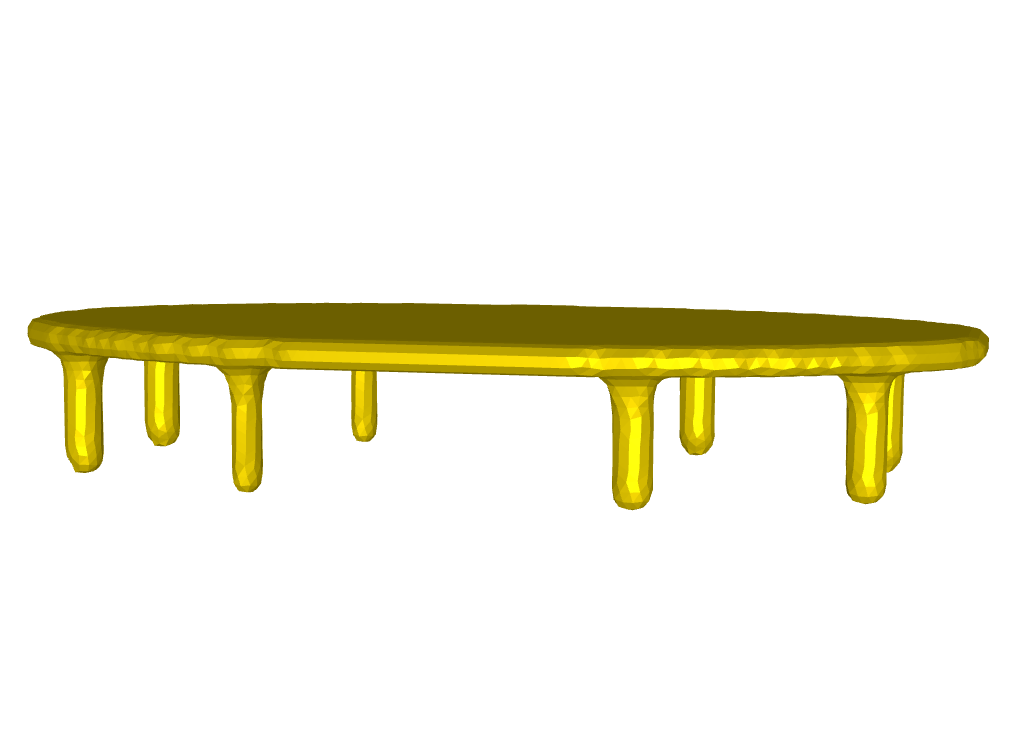}

\caption{Meshes of 3-dimensional point clouds representing from left to right: a human, an octopus and a table.}
\label{fig:3d_shape}
\end{figure}

\begin{figure}[H]

\centering
\includegraphics[width=.3\textwidth]{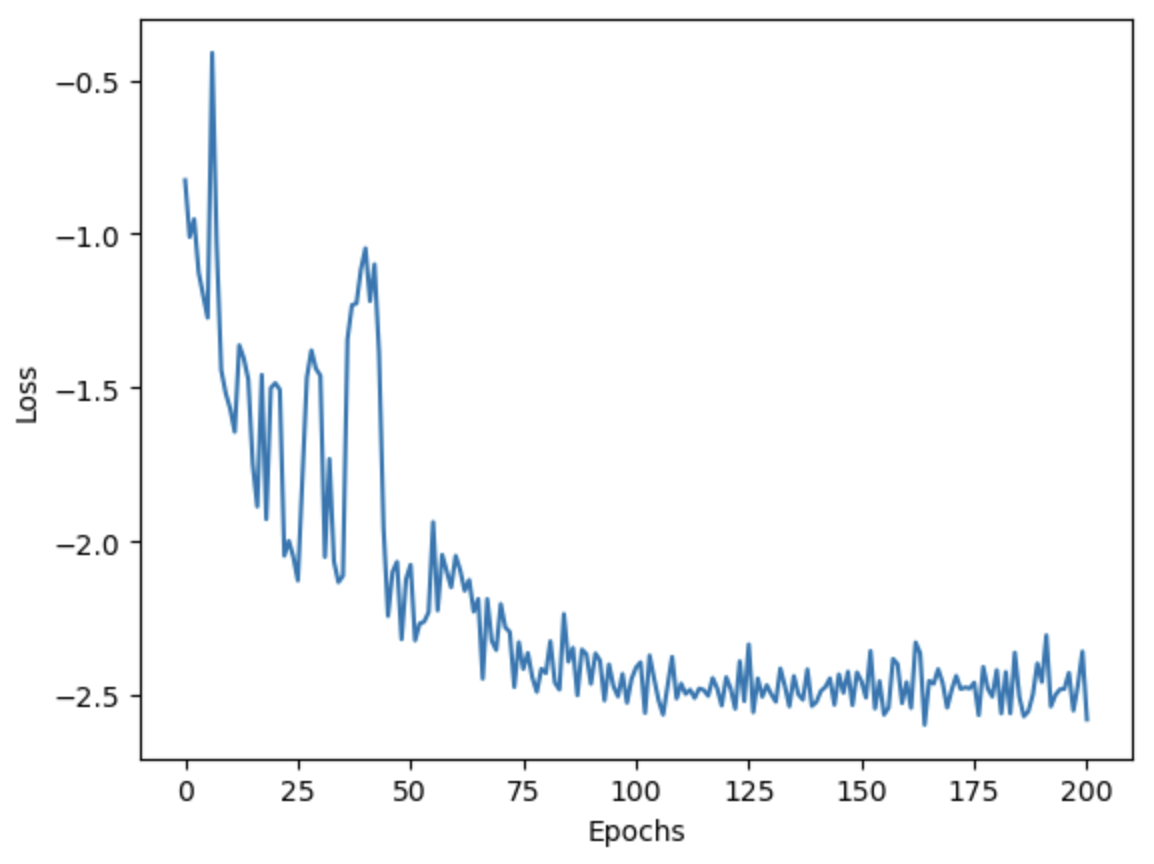}\hfill
\includegraphics[width=.3\textwidth]{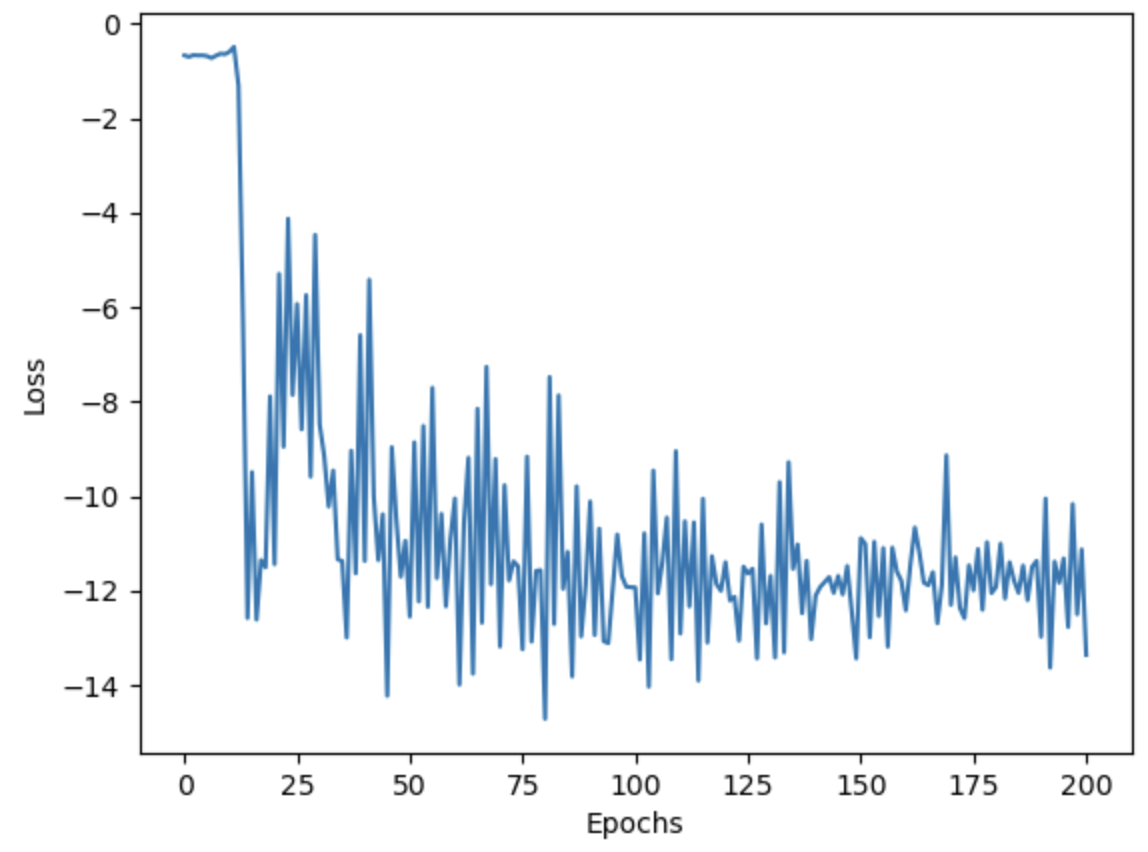}\hfill
\includegraphics[width=.3\textwidth]{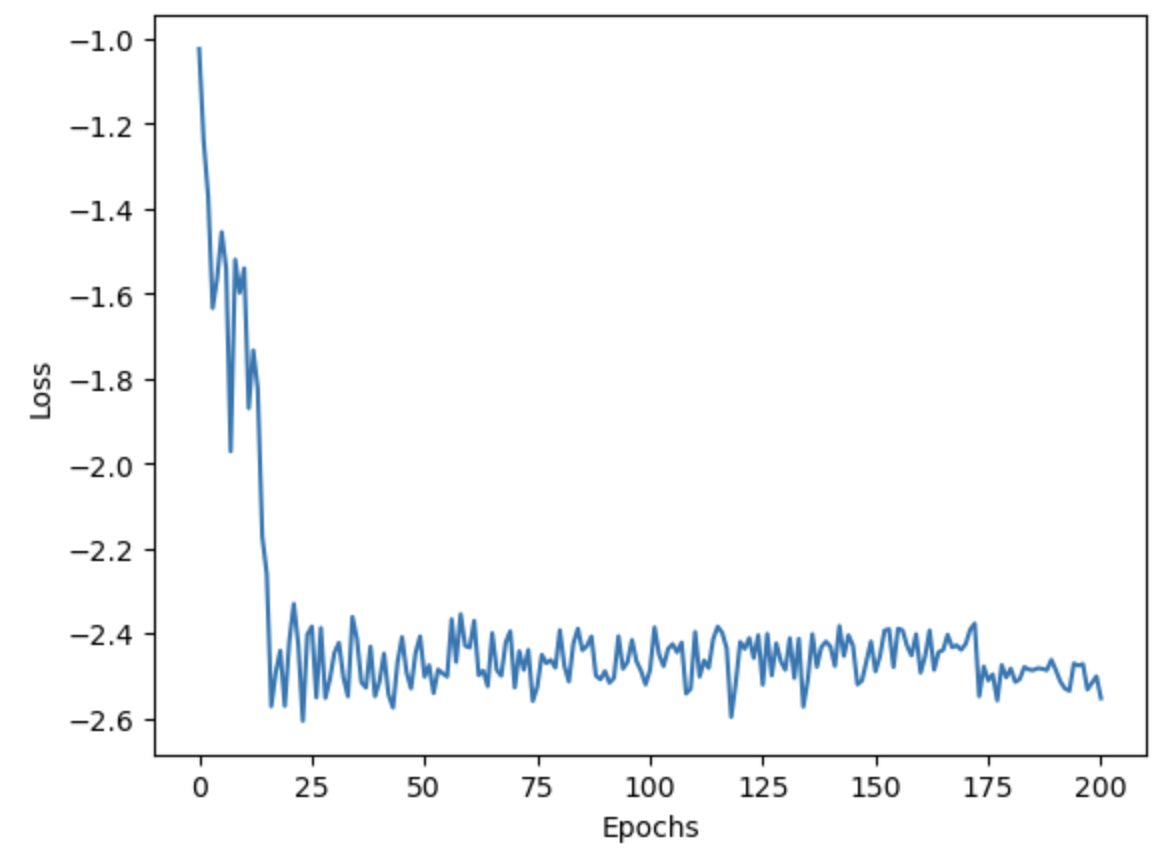}

\caption{Learning curves for the 3-dimensional shapes corresponding, from left to right, to: the human, the octopus and the table.}
\label{fig:3d_curves}

\end{figure}

\begin{table}[H]
    \centering
    \begin{tabular}{|c|c|c|}
    \hline
     Human    &  Octopus & Table\\
     \hline
     0.9999    & -0.9984 & 0.9993\\
     \hline
    \end{tabular}
    \caption{Correlation between the optimized directions and the optimal ones, for each 3-dimensional shape.}
    \label{tab:3d_correlations}
\end{table}

\begin{figure}[H]

\centering
\includegraphics[width=.3\textwidth]{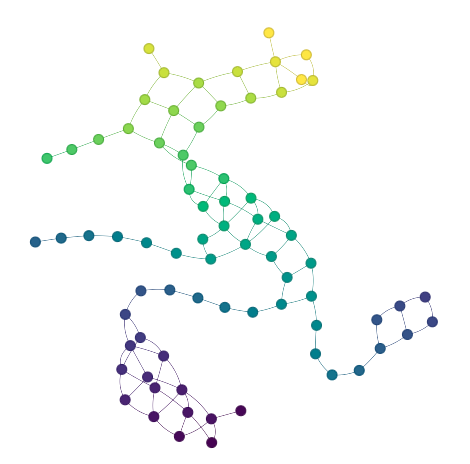}\hfill
\includegraphics[width=.3\textwidth]{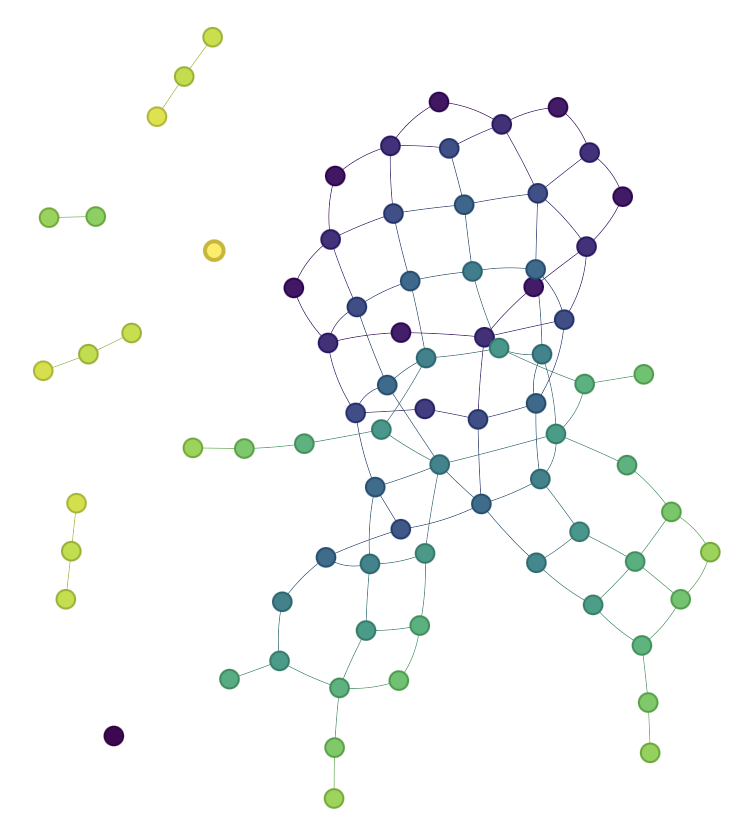}\hfill
\includegraphics[width=.3\textwidth]{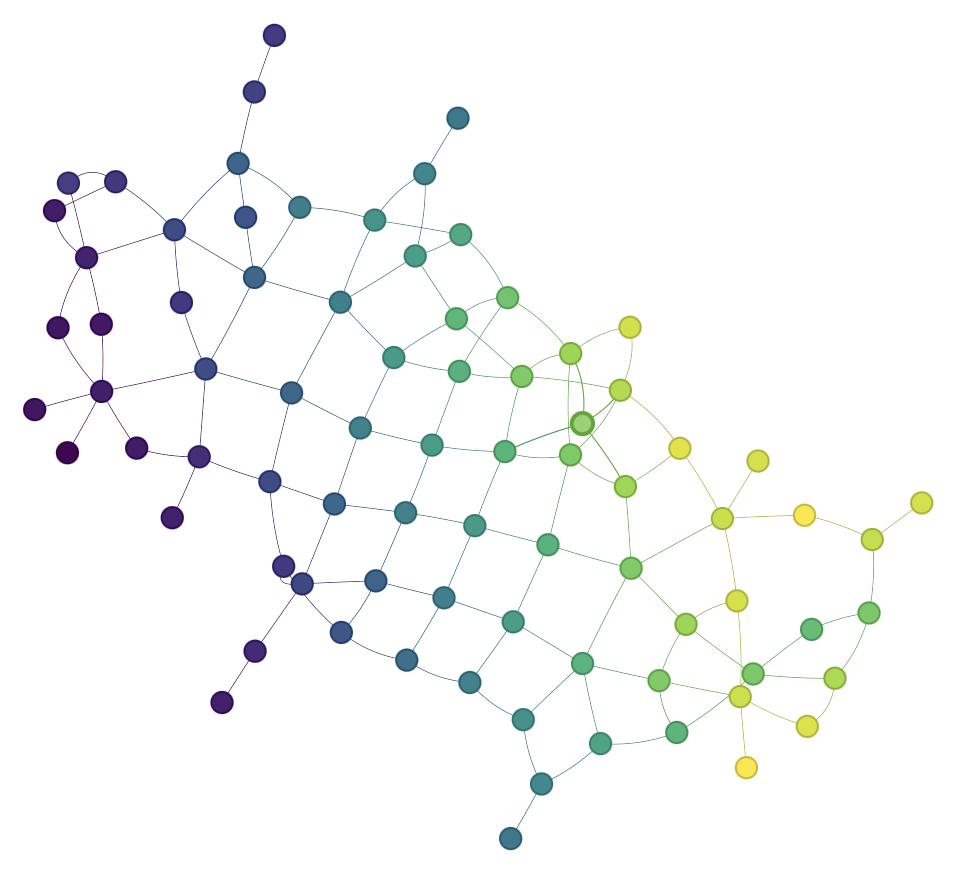}

\caption{Regular Mapper graphs computed with the initial filter function, corresponding, from left to right, to: the human, the octopus and the table. Vertices are colored using the mean value of the filter function in the corresponding clusters.}
\label{fig:3d_initial}

\end{figure}

\begin{figure}[H]

\centering
\includegraphics[width=.3\textwidth]{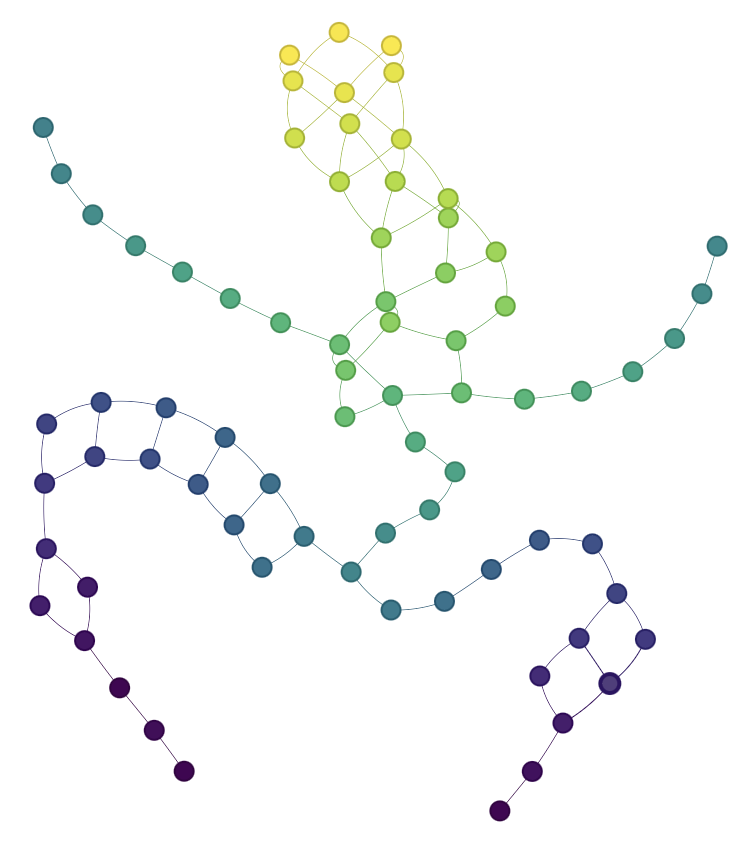}\hfill
\includegraphics[width=.3\textwidth]{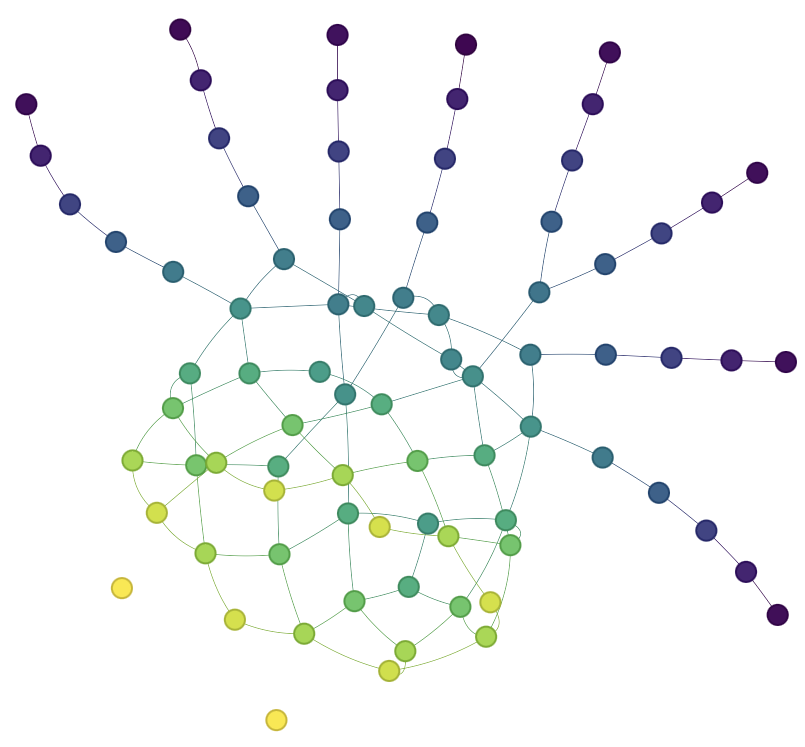}\hfill
\includegraphics[width=.3\textwidth]{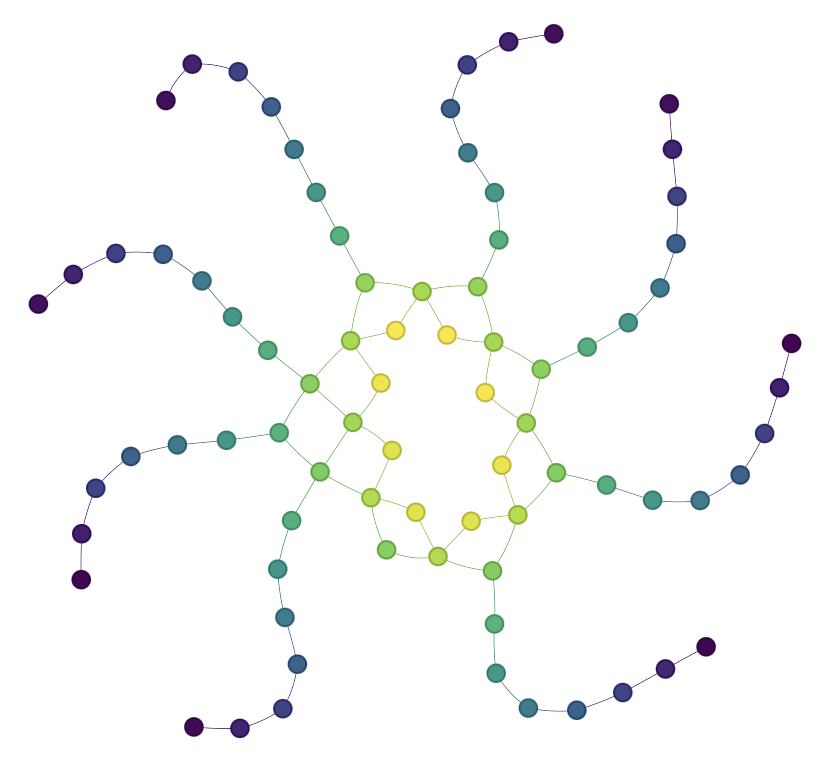}

\caption{Regular Mapper graphs computed with the optimized filter function, corresponding, from left to right, to: the human, the octopus and the table. Vertices are colored using the mean value of the filter function in the corresponding clusters.}
\label{fig:3d_final}

\end{figure}

\subsection{Mapper optimization on RNA-sequencing data}
\label{subsec:single-cell}

%\paragraph*{Human preimplantation dataset.}
We now apply Mapper optimization on the human preimplantation dataset of \cite{petropoulos2016single}, which can also be found in the tutorial of the \texttt{scTDA} Python library. The dataset consists of $n=1,529$ cells form $88$ human preimplantation embryos, sampled at $5$ different timepoints. The dataset can be accessed in the following link \cite{sctda}, and it contains the expression levels for $p=26,270$ genes for each individual cell. The information of the sampling timepoint for each cell is also given, but we do not include it during optimization. The dataset is first preprocessed using the \texttt{Seurat} package in \texttt{R} (gene counts for each cell are divided by the total counts for that cell and multiplied by $10^4$, and then they are natural-log transformed using $\log(1+\cdot)$), which produces a normalized dataset $\Xs_n\subseteq\mathbb{R}^p$.
The parametric family of filter functions we wish to optimize is also linear here, i.e. equal to $\{f_\theta\colon x\mapsto\langle x,\theta \rangle,\,\theta\in\mathbb{R}^p\}$, and the cover assignment scheme $A_\delta$ is the smooth relaxation of the standard case with $\delta=10^{-5}\cdot(\max_{x\in\Xs_n}f_\theta(x)-\min_{x\in\Xs_n}f_\theta(x))$. The cover of the image space is given by $25$ intervals of the same length, such that consecutive intervals have a percentage of $30\%$ of their length in common. The clustering algorithm used is agglomerative clustering and its threshold is fixed using a Hausdorff distance heuristic: we first compute the Hausdorff distance between $\Xs_n$ and a randomly sampled subset of $\Xs_n$ of size $ n/3\approx500$, then we manually tune the threshold using factors of this distance until we get Mapper graphs of reasonable size.

\paragraph*{Objective.} To find $\bar\theta$, we use the total extended persistence as a persistence specific loss $\loss$ and we run Algorithm \ref{alg1} with $N=200$ and $M=10$, taking the diagonal as the initial direction, i.e. $\theta_0=(\frac{1}{\sqrt{p}},...,\frac{1}{\sqrt{p}})^T$. The learning curve in displayed in Figure \ref{fig:sctda_curve} of Appendix~\ref{apdx:experiments}.
The regular Mapper graphs computed using the initial and the final filter functions are displayed in Figure \ref{fig:sctda_mappers}, and are colored with respect to the time component (which was not included in the training dataset).

\begin{figure}[H]
    \centering
    \includegraphics[width=.3\textwidth]{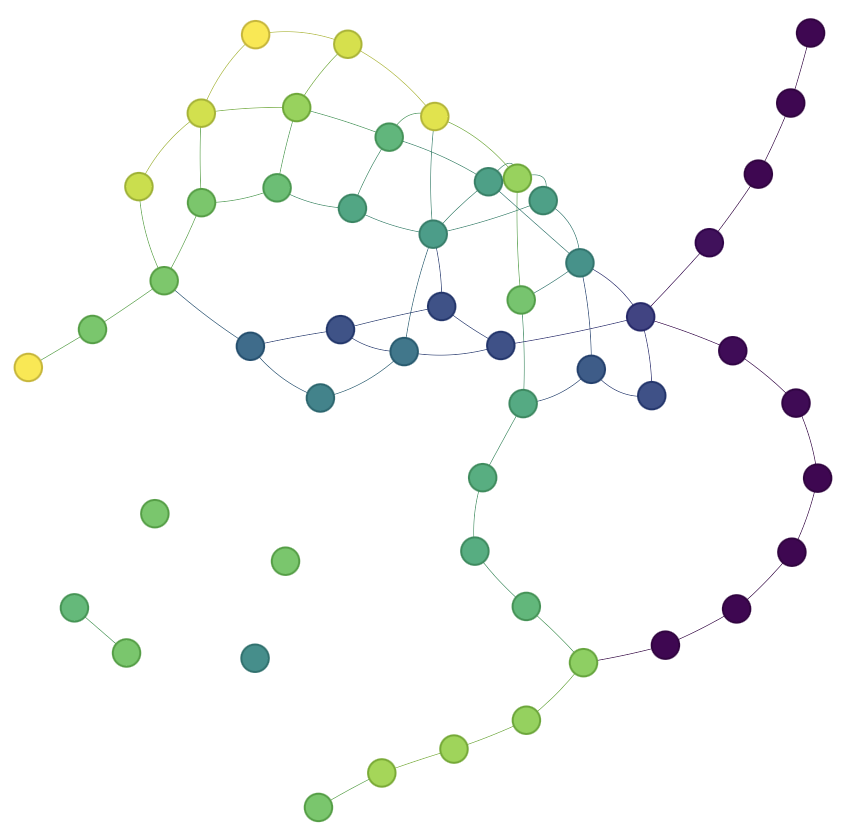}\hspace{0.05\textwidth}
    \includegraphics[width=.3\textwidth]{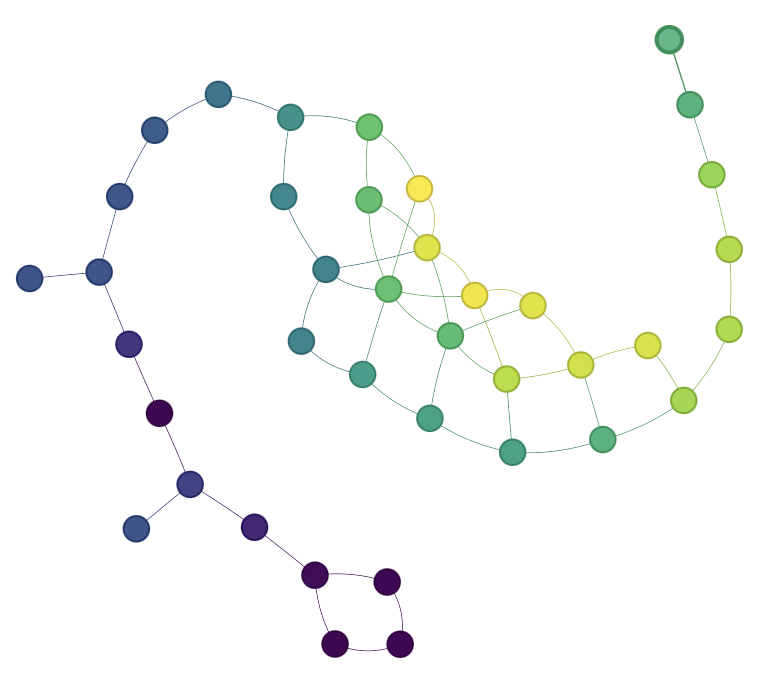}
    \caption{Regular Mapper graphs for the human preimplantation dataset computed using: in the left the initial filter function and in the right the optimized filter function. Vertices are colored using the mean value of the sampling timepoint in the clusters.}
    \label{fig:sctda_mappers}
\end{figure}

\paragraph*{Qualitative assessment.} One can see that the data representation in the Mapper graph produced by the optimized filter function fits the time structure better than with the initial function. In order to confirm this, we isolate each subset of cells having the same sampling timepoint and we plot their respective estimated densities with respect to the initial and the optimized filter function values, in Figure \ref{fig:sctda_densities}. One can see that the optimized filter that we captured is capable of sorting the cells with respect to time, especially at the early timepoints. The reduced performance in this aspect for the later timepoints is, in our guess, due to slowing down of the cell differentiation process. Furthermore, the comparison, in Table \ref{tab:sctda_pearsonr} of Appendix~\ref{apdx:experiments}, between Pearson's correlation coefficients also show that the optimized filter is more correlated to time. \\
%In \cite{rizvi2017single}, a Mapper graph built on the same dataset, where the filters are the first two PCA components, is used. To extract a pseudo-time from the Mapper graph, correlations, between the time and graph distances defined on the different nodes, are first computed. This requires prior knowledge of the time component, whereas our approach uses only the gene expression matrix.
We also verify that the branches in our optimized Mapper graph correspond to the same two genes, HTR3E for the early timepoints and CDX1 for the later ones, that were identified by \cite{rizvi2017single}, see Figure \ref{fig:sctda_genes}. We also identified a few nodes in the branch containing the cells which were sampled in the early stages, that do not contain a high expression level for the HTR3E gene. This could potentially point out another subpopulation of cells with distinct genomic profiles, and that our optimized Mapper graph has captured. An additional experiment using single cell RNA-sequencing data is given in Appendix \ref{apdx:mef}.
 
\begin{figure}[H]
    \centering
    \includegraphics[width=.3\textwidth]{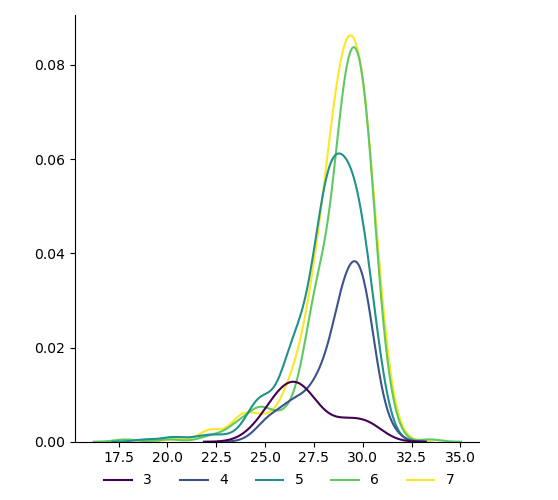}
    \includegraphics[width=.3\textwidth]{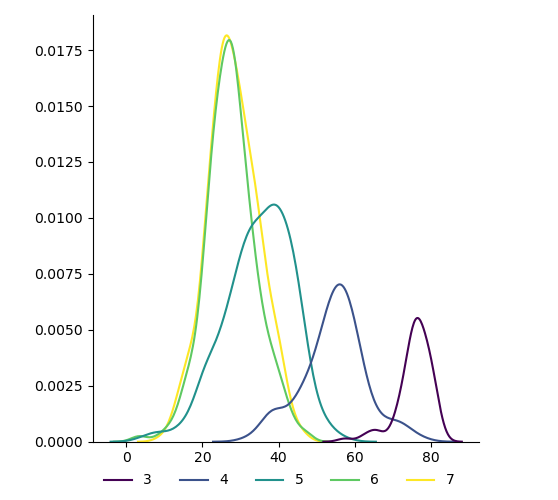}
    \caption{Estimated density of each subset of cells having the same sampling timepoint, with respect to: in the left the initial filter function values and in the right the optimized filter function values. Colors indicate the sampling timepoint in days.}
    \label{fig:sctda_densities}
\end{figure}
 
\begin{figure}[H]
    \centering
    \includegraphics[width=.3\textwidth]{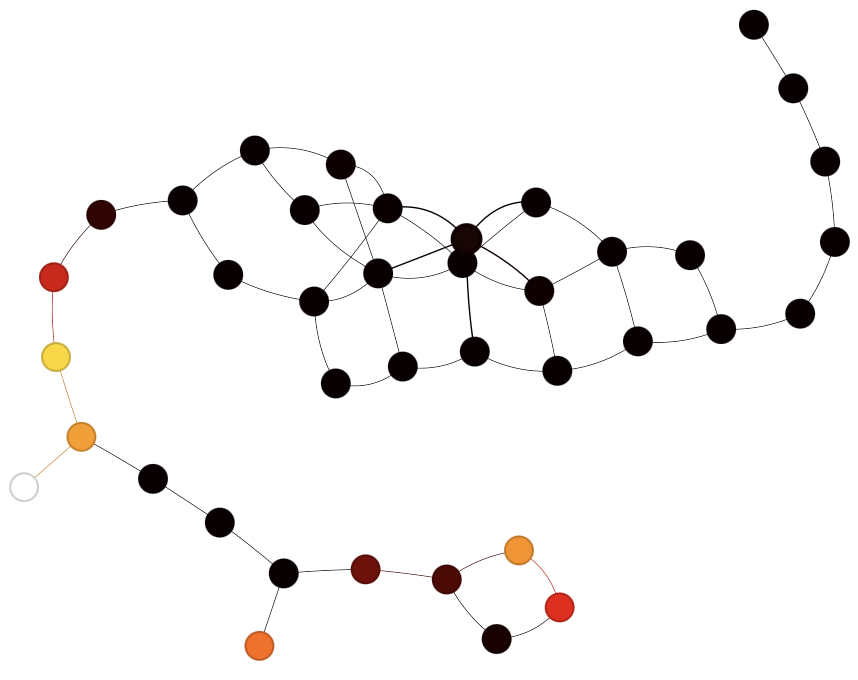}\hspace{0.05\textwidth}
    \includegraphics[width=.3\textwidth]{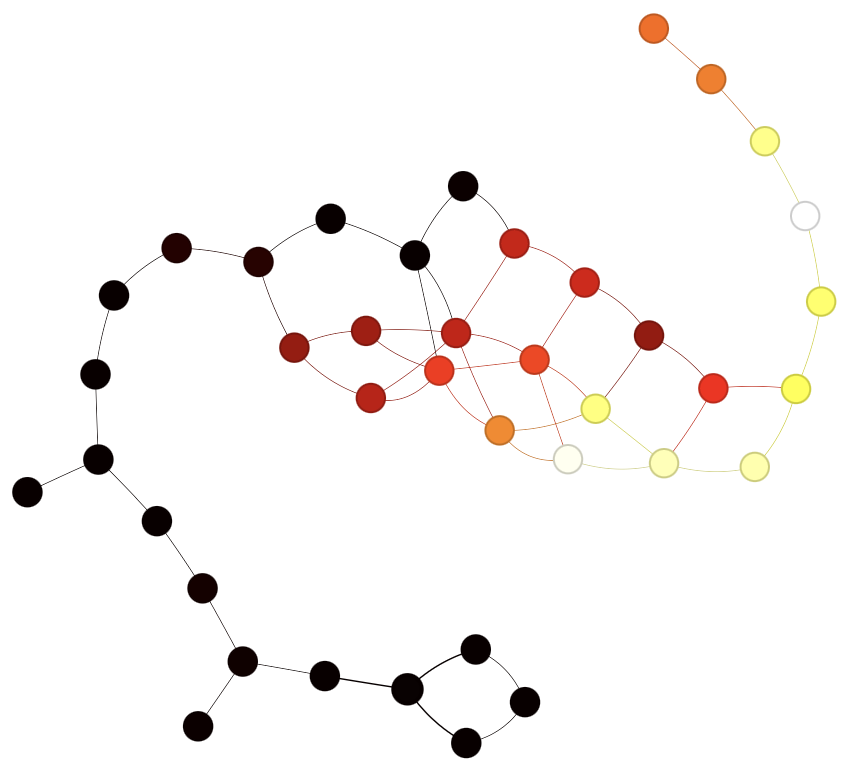}
    \caption{Regular Mapper graph computed using the optimized filter function, colored using the mean normalized expression of: in the left gene HTR3E and in the right gene CDX1.}
    \label{fig:sctda_genes}
\end{figure}

\section{Discussion and future work} \label{sec:discussion}
In this article, we have introduced Soft Mapper, a distributional and smoother version of the standard Mapper graph, with provable convergence guarantees using persistence-based losses and risks. Our case study in this article was finding an optimal filter function, among a parameterized family of functions, in order to construct regular Mapper graphs incorporating an optimized and maximal amount of topological information. 
We then produced examples of such optimization processes on real 3D shape and single-cell RNA sequencing data, for which we were able to obtain structured Mapper representations in an unsupervised way.
%To achieve this, we used persistent homology, to produce a topological loss, and a specific Soft Mapper that can be seen a smooth relaxation of the regular Mapper. We then used this method to produce data representations based on linear filters. 
These representations, especially for the single cell RNA-sequencing data, are not meant to represent novel or state of the art data representations in their respective research domains, but as a proof of concept of the practical benefit of our method.
Moreover, our construction is not limited to the choice of a linear family of filter functions or to the filter optimization setting as a whole. Possible future work includes optimizing non-linear filter functions, based on neural networks or kernel methods, and studying Soft Mappers based on different cover assignment schemes, like the Gaussian cover assignment scheme defined in Appendix \ref{apdx:gaussian}. 

\section{Acknowledgments}

The research was supported by three grants from Agence Nationale de la Recherche: GeoDSIC ANR-22-CE40-0007, ANR JCJC TopModel ANR-23-CE23-0014 and AI4scMed, France 2030 ANR-22-PESN-000.

\newpage
\bibliography{references}
\bibliographystyle{unsrt}

\newpage
\appendix

\section{Gaussian cover assignment scheme}\label{apdx:gaussian}
In this section, it is assumed that  $\Xs_n$ is a point cloud in $\R^p$. Additionally, we consider $r$ centers $\{c_1,...,c_r\}\subseteq\R^p$ and $r$ symmetric, semi-definite and positive matrices $\{\Sigma_1,...,\Sigma_r\}\subseteq\mathbb{R}^{p\times p}$. For each $j\in\{1,...,r\}$, consider the function:
\begin{align*}
q_j\colon \mathbb{R}^p  &\longrightarrow [0,1]\\
x &\longmapsto \exp \left(-(x-c_j)^T  \Sigma_j^{-1} (x-c_j)\right).
\end{align*}
\noindent
Define $A=(A_{i,j})_{\substack{1\leq i\leq n \\ 1\leq j\leq r}}$ to be a random variable in $\{0,1\}^{n\times r}$ such that for every $(i,j)\in\{1,...,n\}\times\{1,...,r\}$~:
$$ A_{i,j}\mid \Xs_n\sim \mathcal{B}(q_j(x_i)),$$
and  as before we take the $ A_{i,j}$'s to be jointly conditionally independent.

This cover assignment scheme is similar to Gaussian mixture models, in that its realizations can be seen as a "one-hot encoding" of the latent variables in a mixture model. However, we can see that a realization of $A$ can have more than one non-zero entry per line as opposed to a mixture model. Furthermore, mean and variance parameters can be inferred with an EM algorithm, and estimated proportions can be also involved in the definition  of the $q_j$'s. 

Note that this strategy of defining a cover assignment scheme does not use a filter function or an overlapping cover entirely.

\section{Elements of o-minimal geometry}\label{apdx:o-minimal}

\begin{definition} An o-minimal structure on the field of real numbers $\mathbb{R}$ is a collection $(S_n)_{n\in\mathbb{N}}$ where each $S_n$ is a set of subsets of $\mathbb{R}^n$ that satisfies:
\begin{enumerate}
    \item  All algebraic subsets of $\mathbb{R}^n$ are in $S_n$;
    \item  $S_n$ is a Boolean subalgebra of the powerset of $\mathbb{R}^n$ (i.e. stable by finite union, finite intersection and complementary);
    \item  if $A\in S_n$ and $B\in S_m$, then $A\times B\in S_{n+m}$;
    \item  if $\pi\colon\mathbb{R}^{n+1}\to\mathbb{R}^n$ is the linear projection onto the first $n$ coordinates and $A\in S_{n+1}$ then $\pi(A)\in S_n$;
    \item  $S_1$ is exactly the family of finite unions of points and intervals.
\end{enumerate}
\end{definition}

The elementary example of an o-minimal structure is the collection of semi-algebraic sets. An element $A\in S_n$ for some $n\in\mathbb{N}$ is called a definable set. Furthermore, a map $f \colon  A\to \mathbb{R}^m$ is called a definable map if its graph (i.e. $\{(x,f(x))\,:\,x\in A\}$) is in $S_{n+m}$. 

Definable maps are stable under addition, product and composition. A function that is coordinate-wise definable is also definable. Moreover, the result of \cite{wilkie1996model} shows that there exists an o-minimal structure that contains the graph of the exponential function. 

An important property of definable maps is that they admit a finite Whitney stratification. This means that if $f \colon  A\to \mathbb{R}^m$ is definable with $A\in S_n$, then $A$ can be decomposed into a finite union of smooth manifolds such that the restriction of $f$ to each of these manifolds is a smooth function. 

For more details on o-minimal geometry, see \cite{coste2000introduction}.

\section{Proof of Theorem \ref{thm1}}\label{apdx:proof}

\begin{lemma}\label{lem1}
Let $\mathcal{S}$ be an o-minimal structure on $\R$. Assume that the two following conditions are satisfied. 
    \begin{itemize}
        \item For every $x\in\Xs_n$, the function $\theta \in \ParSpace \mapsto f_\theta(x)$ is definable in $\mathcal{S}$ and is locally Lipschitz.
        \item For every $m\in\mathbb{N}$, the restriction of the persistence specific loss $\loss$ to the set of persistent diagrams of size $m$ 
%(by isomorphy considered to be $\mathbb{R}^m$) 
is definable in $\mathcal{S}$ and is locally Lipschitz. 
    \end{itemize}
Then for every $e\in{\{0,1\}}^{n\times r}$, the function 
$$\mathcal{L}_e\colon\theta \in \ParSpace \mapsto \mathcal{L}(e,f_\theta)$$
is definable in $\mathcal{S}$ and is locally Lipschitz.
\end{lemma}

\begin{proof}  Let $e\in{\{0,1\}}^{n\times r}$. Let $K=\MapComp(e)$ with vertex set $V$.
%, and thus $K$ is fixed for a given $e$. 
Remember that each vertex $c \in V$ is actually a subset of $\Xs_n$. We now define three maps to decompose the function $\mathcal{L}_e$. First, let us introduce the function     
\begin{align*}
\VertexFilt \colon \ParSpace  &\longrightarrow \mathbb{R}^{\mid V\mid}\\
\theta&\longmapsto \left(\frac{\sum_{x\in c} f_\theta(x) }{\text{card}(c)}\right)_{c\in V}.
\end{align*}
For each coordinate of the function $\VertexFilt$, that is for each $c \in V$, the function $\theta \longmapsto [\VertexFilt (\theta)]_c$  
is a linear combination of the functions $\theta\mapsto f_\theta(x)$. We can therefore see that it is definable in $\mathcal{S}$ and locally Lipschitz, by our first assumption. 

Then we introduce 
\begin{align*}
\SubFilt\colon \mathbb{R}^{\mid V\mid}  &\longrightarrow \mathbb{R}^{\mid K\mid}\\
\Phi&\longmapsto (\max_{c\in\sigma}\Phi_c)_{\sigma\in K},
\end{align*}
and finally $\Persistence \colon \mathbb{R}^{\mid K\mid}   \longrightarrow \mathbb{R}^{\mid K\mid} $ that computes persistence for a filtration that acts on a fixed simplicial complex. The two functions $\SubFilt$ and $\Persistence$  are taken from \cite{carriere2021optimizing}, where they are both proven to be definable in every o-minimal structure and Lipschitz.

Notice that:
$$\mathcal{L}_{e}=\loss \circ\Persistence\circ\SubFilt\circ\VertexFilt.$$
Since $e$, and thus $K$, are fixed, $\loss$ can be replaced by its restriction to persistence diagrams of size $\vert K\vert$.
Hence, following our second assumption, $\mathcal{L}_{e}$ is definable in $\mathcal{S}$ and locally Lipschitz.
\end{proof}

Recall the assumptions in Theorem \ref{thm1}~:

    Suppose that there exists an o-minimal structure $\mathcal{S}$ and we have that: 
    \begin{itemize}
        \item for every $x\in\Xs_n$, the function $\theta\mapsto f_\theta(x)$ is definable in $\mathcal{S}$ and is locally Lipschitz.
        \item for every $m\in\mathbb{N}$, the restriction of $\loss$ to the set of persistent diagrams of size $m$ 
%(by isomorphy considered to be $\mathbb{R}^m$) 
is definable in $\mathcal{S}$ and is locally Lipschitz. 
        \item for every $e\in{\{0,1\}}^{n\times r}$, the function $\theta\mapsto \mathbb{P}_\theta(A =e| \Xs_n)$ is definable in $\mathcal{S}$ and is locally Lipschitz.

    \end{itemize}

By Lemma \ref{lem1} and following the first two assumptions, we know that for every $e\in{\{0,1\}}^{n\times r}$, the function $\mathcal{L}_e\colon\theta\mapsto \mathcal{L}(e,f_\theta)$ is definable in $\mathcal{S}$ and is locally Lipschitz. Now, for every $\theta\in\ParSpace$:
$$\text{L}(\theta)=\sum_{e\in {\{0,1\}}^{e\times r}}\mathcal{L}(e,f_\theta)\cdot\mathbb{P}_\theta(A=e | \Xs_n).$$
As such, L is a sum of products between functions that are definable in $\mathcal{S}$ and locally Lipschitz. We conclude that L is itself definable in $\mathcal{S}$ and locally Lipschitz.

Note that the local Lipschitz property is stable by product (as opposed to the global Lipschitz property). This is due to the fact that the product of two Lipschitz and bounded functions is Lipschitz, and the fact that we can always limit the neighborhoods of points in $\ParSpace$ to bounded ones.

\section{Technical conditions for Stochastic Gradient Descent}\label{apdx:sgd}
We are in the setting where we use stochastic gradient descent to minimize a function L.
If we write the iterates of the algorithm as: 
$$x_{k+1}=x_k-\alpha_k(y_k+\xi_k),$$
where 
$$y_k\in\text{Conv}\left\{\lim_{z\to x_k}\nabla\text{L}(z)\,:\,\text{L is differentiable at }z\right\},$$
consider the following three conditions: 
\begin{enumerate}
    \item for any $k$, $\alpha_k\geq 0$, $\sum_{k=1}^\infty \alpha_k=+\infty$ and $\sum_{k=1}^\infty \alpha_k^2<+\infty$;
    \item $\sup_k\Vert x_k\Vert<+\infty$, almost surely;
    \item Conditionally on the past, $\xi_k$ must have zero mean and have a second moment that is bounded by a function $p\colon\ParSpace\longrightarrow\mathbb{R}$ which is bounded on bounded sets. 
\end{enumerate}
Note that the last condition is satisfied by taking a sequence of independent and centered variables with bounded variance, which are also independent of the past iterates $(x_k)_k$ and  $(y_k)_k$.

According to \cite{davis2020stochastic}, under these three conditions together with the condition that L is definable in an o-minimal structure and locally Lipschitz, then $(\text{L}(x_k))_k$ converges almost surely to a critical values and the limit points of $(x_k)_k$ are critical points of L.

\section{Additional Figures and Tables for the experiments}\label{apdx:experiments}

\begin{table}[H]
    \centering
    \begin{tabular}{|c|c|c|c|}
    \hline
     Parameter & Human    &  Octopus & Table\\
     \hline
     Resolution & 25    & 10 & 10\\
     \hline
     Gain & 0.3    & 0.3 & 0.35\\
     \hline
     Number of clusters & 3    & 8 & 8\\
     \hline
    \end{tabular}
    \caption{Resolution, gain and number of clusters parameters that are used to compute the Mapper for each 3-dimensional shape.}
    \label{tab:r_g}
\end{table}

\begin{figure}[H]
    \centering
    \includegraphics[width=.3\textwidth]{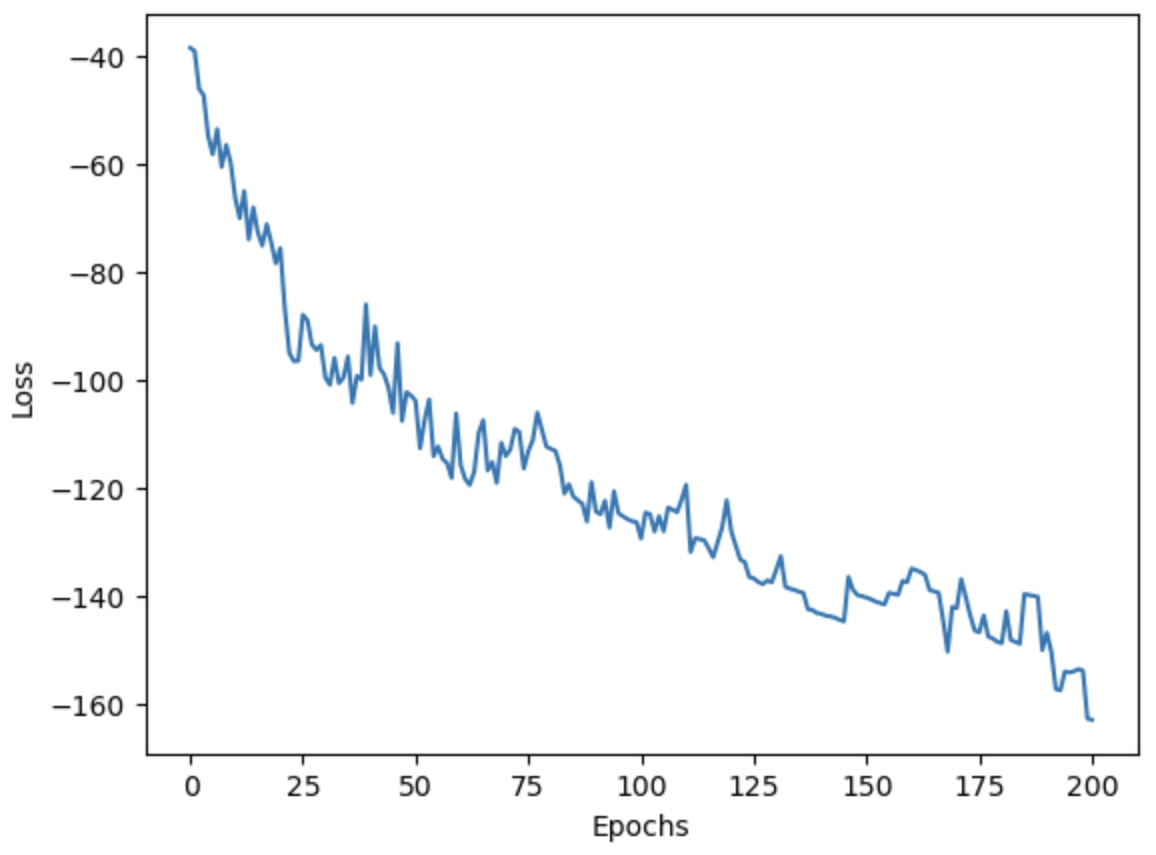}
    \caption{Learning curve for the human preimplantation dataset.}
    \label{fig:sctda_curve}
\end{figure}

\begin{table}[H]
    \centering
    \begin{tabular}{|c|c|c|}
    \hline
    Filter & Correlation with time & P-value\\
    \hline
     Initial    &  0.1330 & 1.7596e-07\\
     \hline
     Optimized    & -0.7549 & 4.0503e-282\\
     \hline
    \end{tabular}
    \caption{Pearson's correlation between the initial filter and time, and the optimized filter and time for the human preimplantation dataset. The associated $p$-values, obtained from testing the null hypothesis that the true correlation coefficient is zero, are also presented.}
    \label{tab:sctda_pearsonr}
\end{table}

\section{Mouse embryonic fibroblasts reprogramming dataset}\label{apdx:mef}

We consider the mouse embryonic fibroblasts (MEF) reprogramming dataset of \cite{schiebinger2019optimal}. It consists of $p=19,089$ gene expressions for $251,203$ MEF cells, densely sampled across $18$ days, with $39$ individual timepoints. The experiment involves adding Doxorubicine (Dox) to the cells on day 0, withdrawing it at day 8, and then transferring them to either a serum-free N2B27 2i medium or maintaining them in serum. 

\paragraph*{Objective.} We would, therefore, want to produce a representation, using our Soft Mapper optimization, that accounts for the time component (like in Section~\ref{subsec:single-cell}) and for the divergence in the treatment that the cells received at day 8.
In order to achieve this, we first take a uniformly sampled subsample of the dataset of size $n=1,500$ and we use the same preprocessing procedure as with the human preimplantation dataset. Similarly, we consider the same settings (linear family of filter functions, smooth cover assignment scheme, agglomerative clustering, diagonal initial parameter set and extended total persistence), and we run Algorithm \ref{alg1} with $N=300$ and $M=10$.
The learning curve is displayed in Figure \ref{fig:ot_curve}. 

\begin{figure}[H]
    \centering
    \includegraphics[width=.3\textwidth]{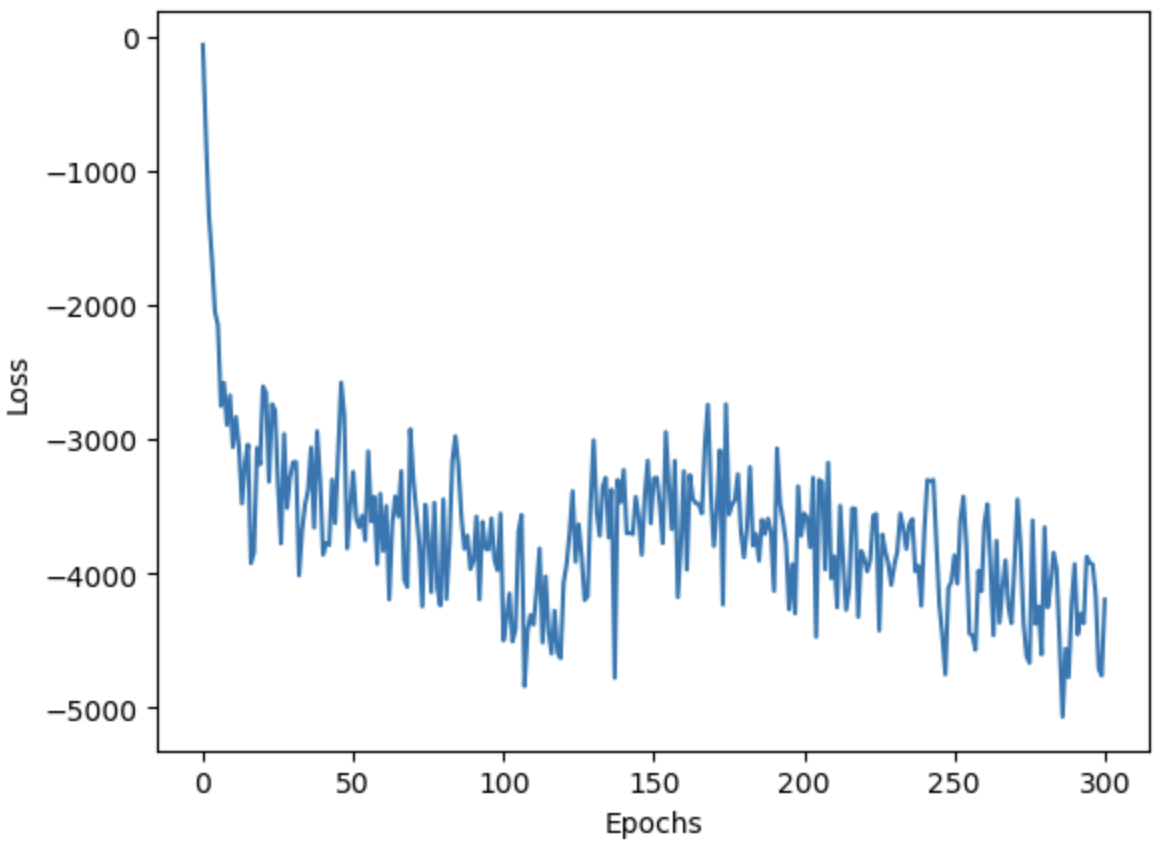}
    \caption{Learning curve for the MEF reprogramming dataset.}
    \label{fig:ot_curve}
\end{figure}

\paragraph*{Qualitative assessment.} By looking at the standard Mapper graphs corresponding to the initial and the optimized filter functions in Figure \ref{fig:ot_mappers}, one can see that the optimized Mapper graph represents the time component better and that it shows two major branches, which point to the two phases that appear in day 8 of the experiment. These observations are confirmed by the improvement in the Pearson's correlation coefficients with respect to time between the initial and the optimized filter function values in Table \ref{tab:ot_pearsonr}. We also color the optimized Mapper graph in Figure \ref{fig:ot_phases} using the three phases in the experiment (Dox, 2i and Serum), each mapped to a different color channel.

\begin{figure}[H]
    \centering
    \includegraphics[width=.3\textwidth]{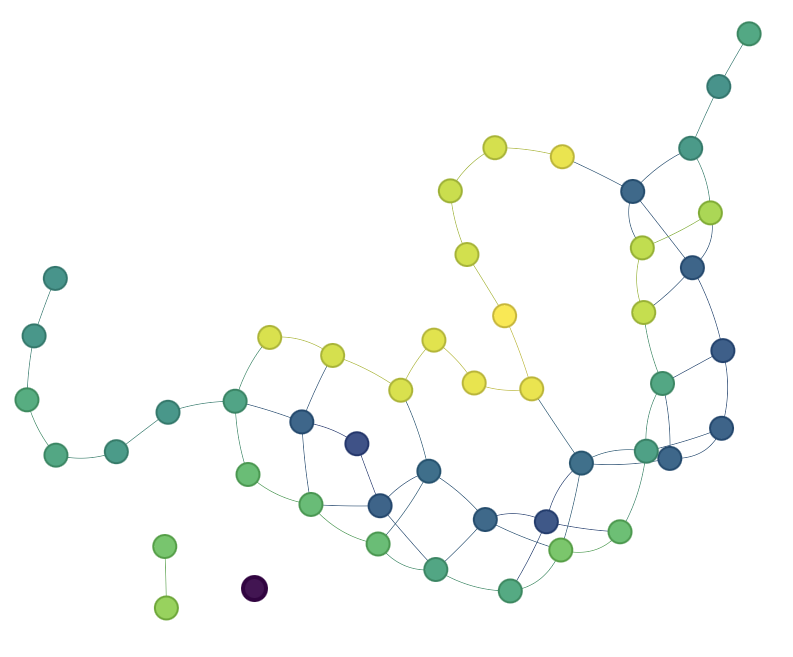}\hfill
    \includegraphics[width=.3\textwidth]{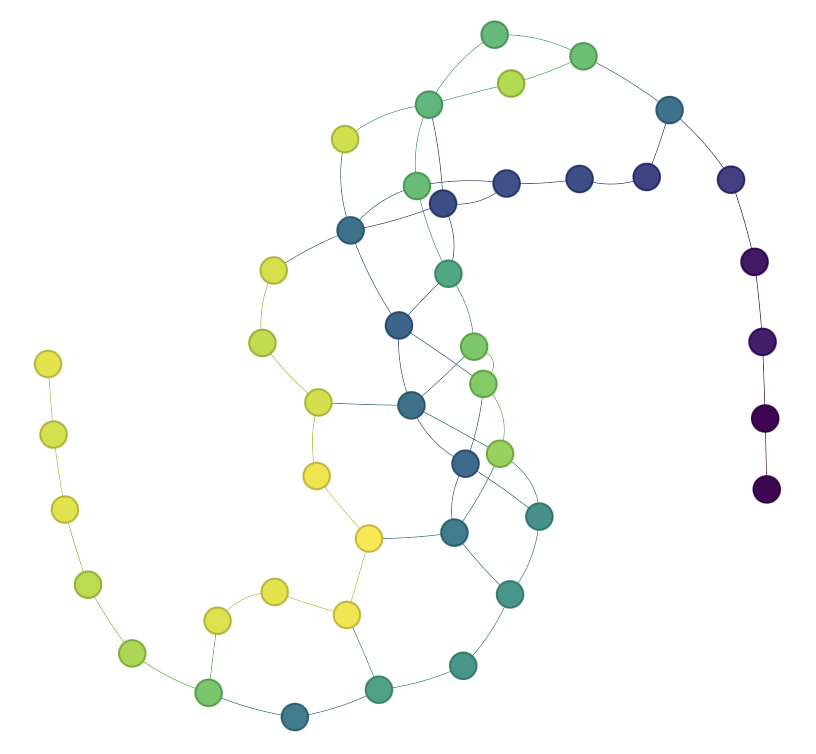}
    \caption{Classical Mapper graphs for the MEF reprogramming dataset computed using: in the left the initial filter function and in the right the optimized filter function. Vertices are colored using the mean value of the sampling timepoint in the corresponding clusters.}
    \label{fig:ot_mappers}
\end{figure}
\begin{table}[H]
    \centering
    \begin{tabular}{|c|c|c|}
    \hline
    Filter & Correlation with time & P-value\\
    \hline
     Initial    &  -0.0560 & 2.9882e-02\\
     \hline
     Optimized    & -0.4015 & 3.2090e-59\\
     \hline
    \end{tabular}
    \caption{Pearson's correlation between the initial filter and time, and the optimized filter and time. The associated $p$-values, obtained from testing the null hypothesis that the true correlation coefficient is zero, are also presented.}
    \label{tab:ot_pearsonr}
\end{table}
\begin{figure}[H]
    \centering
    \includegraphics[width=.3\textwidth]{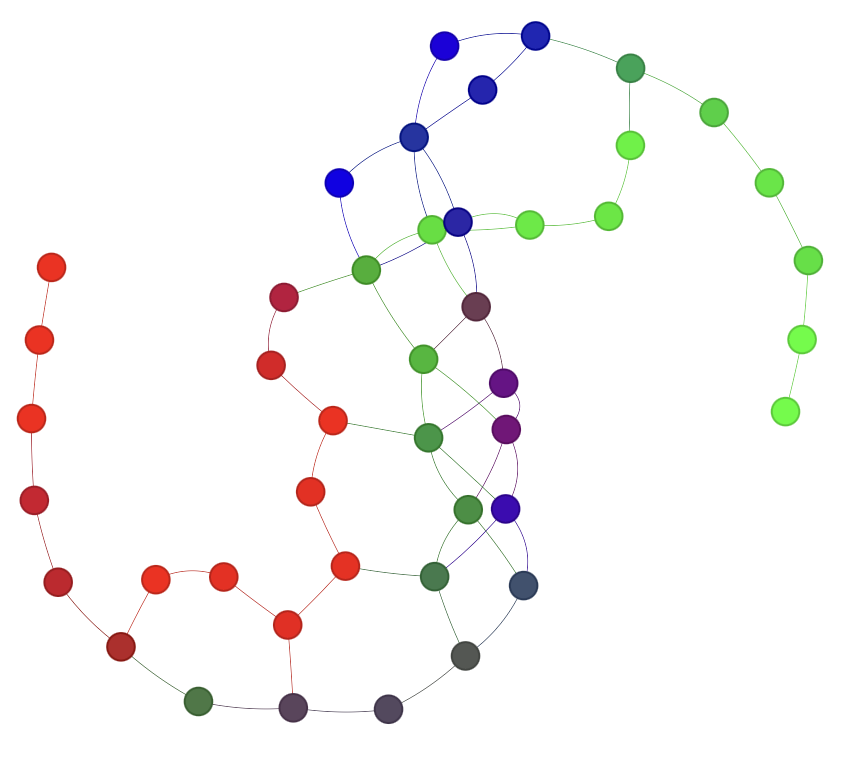}
    \caption{Standard Mapper graph computed using the optimized filter function, colored by mapping each phase to a color channel: Dox in green, Serum in blue and 2i in red.}
    \label{fig:ot_phases}
\end{figure}

\end{document}